\documentclass[journal]{IEEEtran}
\ifCLASSINFOpdf
\else
\fi

\usepackage[pdftex]{graphicx}
\usepackage{cite}
\usepackage{amsmath,amssymb,amsfonts}
\usepackage{algorithmic}
\usepackage{algorithm}
\usepackage{subcaption}
\usepackage{multirow}
\usepackage{booktabs}
\usepackage{dsfont}
\usepackage{mathrsfs}
\usepackage{arydshln}
\usepackage[hyphens]{url}

\newtheorem{definition}{Definition}
\newtheorem{assumption}{\textbf{Assumption}}[section]
\newtheorem{proposition}{\textbf{Proposition}}[section]
\newtheorem{proof}{\textbf{Proof}}[section]
\newtheorem{corollary}{\textbf{Corollary}}[section]
\newtheorem{theorem}{\textbf{Theorem}}[section]

\IEEEoverridecommandlockouts
\usepackage{tikz}
\usepackage{textcomp}
\usepackage{hyperref}
\usepackage{lipsum}

\newcommand\copyrighttext{%
\scriptsize
© 2025 IEEE. Personal use of this material is permitted. Permission from IEEE must be obtained for all other uses, in any current or future  media, including reprinting/republishing this material for advertising or promotional purposes, creating new collective works, for resale or redistribution to servers or lists, or reuse of any copyrighted component of this work in other works. This is the author’s accepted version of the paper: T. Salazar, J. Gama, H. Ara\'ujo and P. H. Abreu, ``Unveiling Group-Specific Distributed Concept Drift: A Fairness Imperative in Federated Learning'', \textit{IEEE Transactions on Neural Networks and Learning Systems}, early access, September 2025. DOI: \href{https://doi.org/10.1109/TNNLS.2025.3601834}{10.1109/TNNLS.2025.3601834}%
}
\newcommand\copyrightnotice{%
\begin{tikzpicture}[remember picture,overlay]
\node[anchor=south,yshift=1pt] at (current page.south) {\fbox{\parbox{\dimexpr\textwidth-\fboxsep-\fboxrule\relax}{\copyrighttext}}};
\end{tikzpicture}%
}

\begin{document}

\title{Unveiling Group-Specific Distributed Concept Drift: A Fairness Imperative in Federated Learning}

\author{Teresa~Salazar,
        Jo\~ao~Gama,
        Helder~Ara\'ujo,
        and Pedro~Henriques~Abreu
\thanks{ 
This work is financed through national funds by FCT - Fundação para a Ciência e a Tecnologia, I.P., in the framework of the Project UIDB/00326/2025 and UIDP/00326/2025. This work was supported in part by the Portuguese Foundation for Science and Technology (FCT) Research Grant 2021.05763.BD. This work was partially supported by the Portuguese Recovery and Resilience Plan (PRR) through project C645008882-00000055, Center for Responsible AI.}%
\thanks{T. Salazar and P. H. Abreu are with CISUC/LASI - Centre for Informatics and Systems of the University of Coimbra, Department of Informatics Engineering, University of Coimbra, 3030-290 Coimbra, Portugal.
E-mail: tmsalazar@dei.uc.pt}%
\thanks{J. Gama is with INESC TEC and the Faculty of Economy, University of Porto, 4200-465 Porto, Portugal.}%
\thanks{H. Ara\'ujo is with the Institute of Systems and Robotics, Department of Electrical and Computer Engineering, University of Coimbra, 3030-290 Coimbra, Portugal.}}

%\author{Michael~Shell,~\IEEEmembership{Member,~IEEE,}
        %John~Doe,~\IEEEmembership{Fellow,~OSA,}
        %and~Jane~Doe,~\IEEEmembership{Life~Fellow,~IEEE}% <-this % %stops a space
%\thanks{M. Shell was with the Department
%of Electrical and Computer Engineering, Georgia Institute of %Technology, Atlanta,


\markboth{IEEE TRANSACTIONS ON NEURAL NETWORKS AND LEARNING SYSTEMS}%
{T. Salazar \MakeLowercase{\textit{et al.}}: Bare Demo of IEEEtran.cls for IEEE Journals}



\maketitle

\copyrightnotice

\begin{abstract}
In the evolving field of machine learning, ensuring group fairness has become a critical concern, prompting the development of algorithms designed to mitigate bias in decision-making processes. Group fairness refers to the principle that a model's decisions should be equitable across different groups defined by sensitive attributes such as gender or race, ensuring that individuals from privileged groups and unprivileged groups are treated fairly and receive similar outcomes. However, achieving fairness in the presence of group-specific concept drift remains an unexplored frontier, and our research represents pioneering efforts in this regard. Group-specific concept drift refers to situations where one group experiences concept drift over time while another does not, leading to a decrease in fairness even if accuracy remains fairly stable. Within the framework of Federated Learning, where clients collaboratively train models, its distributed nature further amplifies these challenges since each client can experience group-specific concept drift independently while still sharing the same underlying concept, creating a complex and dynamic environment for maintaining fairness. The most significant contribution of our research is the formalization and introduction of the problem of group-specific concept drift and its distributed counterpart, shedding light on its critical importance in the field of fairness. Additionally, leveraging insights from prior research, we adapt an existing distributed concept drift adaptation algorithm to tackle group-specific distributed concept drift which uses a multi-model approach, a local group-specific drift detection mechanism, and continuous clustering of models over time. The findings from our experiments highlight the importance of addressing group-specific concept drift and its distributed counterpart to advance fairness in machine learning.
\end{abstract}


% \fi
\IEEEpeerreviewmaketitle

\section{Introduction}
\IEEEPARstart{F}{airness} is a fundamental concern in machine learning, aiming to prevent discrimination in algorithmic decision-making processes with respect to sensitive attributes such as race or gender \cite{survey_fairness_ml, mehrabi2021survey, cruz2025guidelines}. Ensuring group fairness means that models should provide equitable outcomes across different groups defined by these sensitive attributes. For instance, the COMPAS (Correctional Offender Management Profiling for Alternative Sanctions) tool \cite{compas}, used in criminal justice for assessing recidivism risk, has been criticized for exhibiting racial bias. Specifically, it was found to disproportionately label Black defendants as high risk compared to white defendants, despite having similar risk profiles. This case underscores the critical importance of addressing fairness in machine learning systems and highlights the potential consequences of biased algorithms.

However, achieving fairness becomes even more challenging in the presence of concept drift, which refers to the phenomenon where the underlying data distribution changes over time. Concept drift generally impacts overall model accuracy and performance, as models may struggle to adapt to new patterns in the data. Specifically, group-specific concept drift, a concept we introduce and formalize in this paper, occurs when the data distribution for a specific group, characterized by a sensitive attribute, changes over time, while the distribution for other groups remains constant. This discrepancy can have significant implications for fairness, as it directly affects equitable outcomes across different groups.

To illustrate the importance of the problem, consider an example of group-specific concept drift in the context of loan lending. Suppose a financial institution operates in a region where historically marginalized communities have limited access to financial services. Over time, societal changes, government initiatives, and increased awareness promote financial inclusion, leading to an increase in loan applications from these previously underserved communities. In this evolving environment, group-specific concept drift occurs as the data distribution for loan applicants from marginalized communities changes. Traditional lending models may not adequately capture the creditworthiness and financial profiles of these communities due to the historical underrepresentation and lack of relevant data. If the model fails to adapt to these changes, its predictions for these communities may become less accurate, even though the overall accuracy might remain stable. This results in reduced fairness, as the model’s outcomes become inequitable across different groups. Specifically, the decrease in accuracy for the affected group leads to a decline in group fairness. Consequently, there is a risk of perpetuating bias in loan approval decisions. Addressing group-specific concept drift is relevant to ensure fair access to loans for individuals from historically marginalized communities. 

Federated Learning (FL) offers a distinctive environment for studying fairness in the context of group-specific concept drift within a distributed framework. FL is a collaborative machine learning approach where multiple clients train a shared model under the coordination of a central server, while keeping the training data decentralized \cite{fl-definition}. The continuous operation of FL systems over extended periods makes them susceptible to gradual changes such as concept drift. This operational framework is well-suited for examining fairness in the presence of group-specific distributed concept drift, as the distributed nature of FL introduces novel challenges. Specifically, group-specific concept drift may manifest in different clients at different times, complicating efforts to maintain fairness and accuracy across the entire system. As such, in this work, we also address these challenges by introducing and formalizing the concept of group-specific distributed concept drift within distributed environments.

Going back to the previous example of loan lending, now in a FL setting, consider a scenario where clients representing different geographical regions experience varying rates of change in the data distribution of loan applicants. For instance, a financial institution operating in a diverse geographic landscape may have multiple branches, each serving a unique community with its distinct characteristics. In this FL framework, if one branch undergoes changes in the data distribution for loan applicants from marginalized communities, while another branch remains unaffected, the result is a manifestation of group-specific distributed concept drift. The decentralized nature of FL implies that these drifts in data distribution occur independently across different clients at various times, introducing a layer of complexity in maintaining fairness and accuracy. As the financial landscape evolves and societal changes unfold, addressing group-specific distributed concept drift becomes imperative to uphold fairness and prevent perpetuation of biases in loan approval decisions for diverse and historically underserved communities.

\paragraph*{\textbf{Contributions}} Our contributions span three key dimensions, each enhancing our understanding and addressing challenges related to group-specific concept drift in FL:
{\setlength{\parindent}{0cm}
\paragraph*{(1) Formalization of group-specific concept drift and its distributed counterpart}
We pioneer the introduction and formalization of the concept of group-specific concept drift and group-specific distributed concept drift, shedding light on previously unexplored aspects within the fairness landscape.
\paragraph*{(2) Experimental framework for studying group-specific distributed concept drift}
We establish an experimental framework to study this new challenge, leveraging diverse datasets with varying degrees of sensitive group imbalance. This robust setup provides a foundation for exploring the intricacies of group-specific distributed concept drift.
\paragraph*{(3) Algorithm for handling group-specific distributed concept drift}
Building upon techniques from prior studies, we propose FairFedDrift: a continuous multi-model clustering approach with a local group-specific drift detection mechanism to effectively address the intricate problem of group-specific distributed concept drift. Our experimental results demonstrate the efficacy of FairFedDrift in effectively detecting and managing group-specific distributed concept drift in FL.
}

The remainder of this work is structured as follows: Section \ref{sec:related-work} provides some background and related work on the subject, Section \ref{sec:problem-statement} describes the problem statement, Section \ref{sec:fairfeddrift} presents the proposed approach, Section \ref{sec:methodology} describes the methodology and experimental design, Section \ref{sec:experiments} presents the results, Section \ref{sec:discussion} presents the discussions and possible future research, and Section \ref{sec:conclusions} concludes our work.

\section{Related Work}\label{sec:related-work}

\subsection{Group Fairness} 
Fairness-aware machine learning has gained significant attention, and algorithms designed to promote fairness can be categorized into three groups based on the stage at which they are implemented: pre-processing, in-processing, and post-processing \cite{survey_fairness_ml, mehrabi2021survey, fawos, tian2025multifair}. The algorithm proposed in this study belongs to the in-processing category. Fairness metrics can be broadly classified into two main groups: group fairness and individual fairness. Group fairness emphasizes that the outcomes of an algorithm do not disproportionately favor or harm any particular group based on sensitive attributes such as race, gender, ethnicity, religion, or socioeconomic status. On the other hand, individual fairness aims to ensure similar predictions for individuals who share similar characteristics \cite{survey_fairness_ml, fawos}. This research specifically focuses on optimizing group fairness.

\subsection{Fair Federated Learning}
Ensuring fairness in FL is important due to the diverse and heterogeneous nature of both the data and the participants. Several types of fairness have been identified in FL \cite{salazar2024survey}, including group fairness \cite{fair-fate, parecido-group-fairness-fl, parecido-fairl-fl-het-face, parecido-mehrabi, mitigating-bias-fl}, which ensures equitable treatment across sensitive groups; performance fairness \cite{chaudhury2024fair, wang2025fedeba}, which aims for uniform model performance across clients; selection fairness \cite{cho2022towards}, which provides clients an equal chance to participate; and contribution fairness \cite{wang2024fedsac}, which allocates rewards based on each client’s contribution to the global model. Group fairness in FL, which is the focus of this work, aims to mitigate bias while preserving privacy in a decentralized environment using techniques such as fair global model aggregation \cite{fair-fate, parecido-group-fairness-fl, parecido-fairl-fl-het-face, parecido-mehrabi} and local debiasing methodologies at each client \cite{mitigating-bias-fl}. Readers can refer to \cite{salazar2024survey} for a comprehensive survey on this topic.

Current fair FL approaches, designed for static FL, lack mechanisms to detect and adapt to changing data distributions, rendering them unfit for addressing fairness under group-specific distributed concept drift. In such scenarios, characteristics or patterns observed in one concept may not be applicable to others, intensifying fairness challenges in FL environments. Hence, using a multi-model approach offers a promising solution to address the complexities of concept drift while preserving fairness across diverse data distributions.

\subsection{Federated Learning under Concept Drift} 
Concept drift has been extensively studied in the centralized setting, and readers can refer to the surveys \cite{gama2014survey, wang2018systematic} for further details. However, applying centralized algorithms to FL is often unsuitable for handling distributed concept drift, as FL must contend with heterogeneous data distributions across time and clients, all while preserving data privacy. The decentralized nature of FL, where sensitive data remains local and only model updates are shared, poses unique challenges that traditional centralized methods cannot adequately address \cite{jothimurugesan2023federated, chen2021asynchronous, canonaco2021adaptive, casado2022concept, kang2024fednn, guo2024fedrc}. Consequently, most existing algorithms do not fully account for the distributed and privacy-preserving requirements of FL when addressing concept drift.

In the context of FL, research on concept drift is still in its early stages and most works do not account for the distributed nature of concept drift in their algorithms \cite{jothimurugesan2023federated}, failing to address heterogeneity across time and clients simultaneously \cite{chen2021asynchronous, canonaco2021adaptive, casado2022concept, kang2024fednn, guo2024fedrc}. In their recent work, Jothimurugesan et al. \cite{jothimurugesan2023federated} introduce a novel approach to tackle the challenge of distributed concept drift in FL. They treat drift adaptation as a dynamic clustering problem, emphasizing the limitations of single global models when dealing with distributed drifts. To address this issue, the authors propose a multi-model approach that leverages local drift detection based on the overall loss, offering a more robust solution to adapt to evolving data distributions over time. However, it is important to note that this algorithm, referred to as FedDrift, focuses on the global loss and does not consider group-specific losses. In the context of our work, where we specifically address group-specific distributed concept drift, FedDrift's reliance on global loss renders it ill-suited for our scenario. Group-specific concept drift may not be effectively detected by FedDrift, as it does not account for group-specific loss variations in local models that are crucial for capturing changes specific to certain groups. This limitation underscores the need for a more nuanced approach to address fairness concerns under group-specific distributed concept drift, which we elaborate on in the subsequent sections.

\subsection{Limitations and Research Context} While certain studies have delved into fairness under concept drift in centralized streaming environments \cite{iosifidis2019fairness, iosifidis2020online, zhang2019faht, zhang2021farf}, it is important to recognize their limitations when applied to fairness under concept drift in FL. These fairness-aware algorithms may not be well-suited for FL due to several reasons. Firstly, they typically require direct access to sensitive attributes, which is restricted in FL where only model parameters are shared among clients and the data of each client is kept private. Secondly, these algorithms rely on decision trees which may pose a security risk in a FL setting when sharing decision tree models with the server, requiring enhanced privacy and security measures. In addition, these studies primarily operate in a streaming context, with one example arriving at a time, which differs from our setup that is characterized by the simultaneous processing of multiple examples, and contrasts with the typical use of neural networks in FL. Similarly, centralized methods in multistream classification \cite{chandra2016adaptive, haque2017fusion, yu2024online} predict class labels for the target stream using labels from a source stream, where concept drift can occur asynchronously. However, these methods are both fairness-unaware and designed for streaming environments, making them ill-suited for this scenario for the same reasons.

Finally and most importantly, these existing works primarily concentrate on the study of concept drift affecting both groups simultaneously, rather than delving into datasets characterized by group-specific concept drift or its distributed counterpart. This limitation constrains the assessment of their approaches under realistic conditions with evolving fairness dynamics, where different groups may undergo distinct patterns of concept drift. In this work, we endeavor to bridge this gap by pioneering the formalization and exploration of group-specific concept drift in the context of FL, unveiling its importance in advancing fairness in machine learning.

\section{Problem Statement}\label{sec:problem-statement}

We first define the notations used throughout this paper. We assume there exist $K$ clients and $T$ timesteps in the FL setting. In our context, a timestep refers to a specific interval during which new data is received by a client, which is composed of $R$ communication rounds. The collection of private data received by a client $k \in K$ at timestep $t \in T$ is represented as:
\[
D_k^t = \{(x_i, y_i) \mid x_i \in X, y_i \in Y, i = 1, \dots, |D_k^t|\}
\]
Each dataset $D_k^t$ contains $|D_k^t|$ instances. We assume a classification setting, where $X$ is the input space, $Y$ is the output space, $\hat{Y}$ denotes the predicted class, and $S$ represents a sensitive attribute of $X$. The sensitive attribute $S$ takes values $s \in S$, where $S$ is the set of possible values corresponding to different groups.

\subsection{Classical Federated Learning}

The primary goal of FL is to train a central model, $\theta$, located on the server, while preserving the privacy of the clients' data. In FL, multiple clients collaborate to train a model that minimizes the weighted average of the loss across all clients. The objective of this framework can be expressed as follows, as originally proposed in \cite{Federated_Learning_Avg}:

\begin{equation}
\min_{\theta} f(\theta) = \sum_{k \in K} \frac{|D_{k}|}{|D|} G_k(\theta), \qquad G_k(\theta) = \frac{1}{|D_k|} \sum_{i=1}^{n_k} f_{i}(\theta),
\label{eq:min}
\end{equation}
where $G_k(\theta)$ represents the local objective for client $k$, and $f_{i}(\theta)$ represents the loss of data point $i$ from client $k$. It is important to note that this classical formulation does not explicitly account for the temporal variance of data across different timesteps.

This original work \cite{Federated_Learning_Avg} led to the proposal of an algorithm called Federated Averaging (FedAvg) which involves periodically averaging the locally trained models of clients at each communication round to generate the global model. The FedAvg global model at timestep $t$ can be mathematically expressed as:

\begin{equation}
\theta^{t} = \sum_{k \in K} \frac{|D_k|}{|D|} \theta^{t}_k
\end{equation}

This equation represents the aggregation of the locally updated models weighted by the ratio of their respective dataset sizes.

\begin{figure*}[h!]
    \centering
    \includegraphics[width=\linewidth]{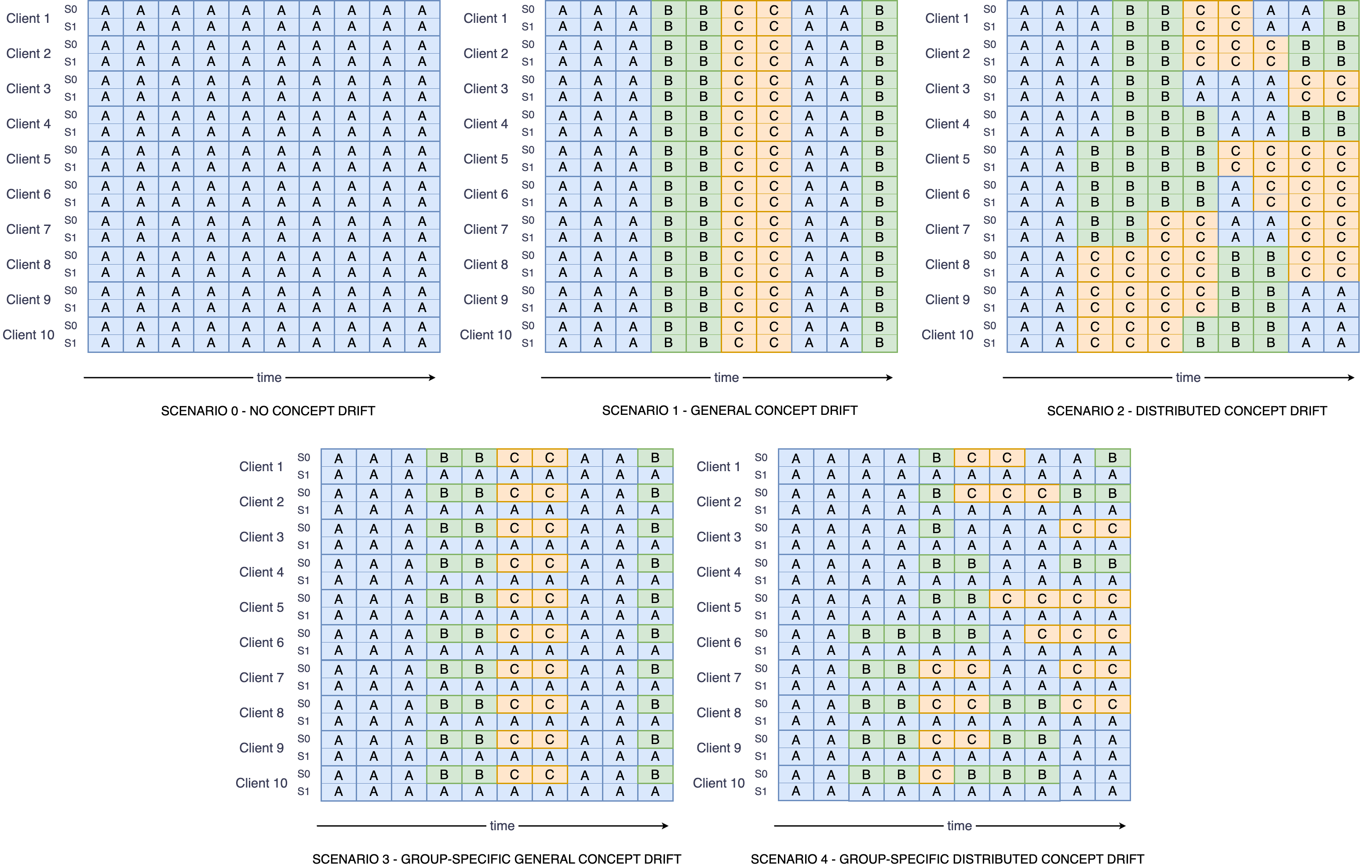}
    \caption{Scenarios of concept drift in Federated Learning: 0 - no concept drift; 1 - general concept drift; 2 - distributed concept drift; 3 - group-specific general concept drift; 4 - group-specific distributed concept drift.}
    \label{fig:concept-drift}
\end{figure*}

\subsection{Concept Drift}

We assume that the underlying distribution of each $D_k$ is not constant over time. This means that the characteristics of the data can vary across time, resulting in concept drifts. Concept drifts refer to changes in the joint distribution, where the distribution $P^t (X, y)$ is different from $P^{t-1} (X, y)$ at two different timesteps \cite{gama2014survey}. We are particularly interested in observing real concept drifts, which occur when the distribution $P^t (y|X)$ is different from $P^{t-1} (y|X)$ and require updating the model.

In this work, we introduce and formalize two distinct types of concept drifts: group-specific concept drift, and group-specific distributed concept drift.

\begin{definition}[\textbf{Group-Specific Concept Drift}]
Let $P^t(y|X, S)$ represent the conditional distribution of $y$ given $X$ and group $S$ at timestep $t$. We first define group-specific concept drift as follows:
\begin{itemize}
    \item For a specific group $S = s$, there is a concept drift if $P^t(y|X, S=s) \neq P^{t-1}(y|X, S=s)$ for any $t > 0$.
    \item For other groups $S \neq s$, there is no concept drift if $P^t(y|X, S \neq s) = P^{t-1}(y|X, S \neq s)$ for any $t > 0$.
\end{itemize}
\end{definition}

In other words, group-specific concept drift occurs when one group experiences a change in the conditional distribution of the target variable over time, while other groups' conditional distribution remain constant. This difference in concept drift between groups can result in a decrease in fairness, even if loss remains fairly stable. This occurs due to changes in loss within group $s$ attributed to concept drift that may not be readily noticeable in the overall loss, particularly when there is an imbalance in the dataset with respect to the sensitive attribute (i.e. when group $s$ is underrepresented).

In these situations, conventional drift detection mechanisms that rely solely on overall loss metrics may struggle to capture these underlying dynamics. This underscores the need to incorporate group-specific loss monitoring over time, as fairness concerns can arise when concept drifts affect opportunities for favorable outcomes across groups. It is important to note that, although studies have delved into fairness under concept drift in centralized environments, the concept of group-specific concept drift had not been formally addressed in the literature. In addition, empirical experiments to validate its impact are lacking as the literature focuses on general concept drift, overlooking the nuances introduced by group-specific variations.

\begin{definition}[\textbf{Group-Specific Distributed Concept Drift}]
Let $P^t_k(y|X, S)$ represent the conditional distribution of $y$ given $X$ for client $k$ at timestep $t$ and group $S$. In FL, we first define group-specific distributed concept drift as follows:
\begin{itemize}
    \item For each client $k$, there can be a group-specific concept drift i.e., $P^t_k(y|X, S=s) \neq P^{t-1}_k(y|X, S=s)$ and $P^t_k(y|X, S \neq s) = P^{t-1}_k(y|X, S \neq s)$ for any $t > 0$.
    \item The conditional distribution for client $k$ may differ from or be the same as the conditional distribution for client $k'$ at either the same timestep $t$ or a different timestep $t'$ i.e. $P^t_k(y|X, S) = P^{t'}_{k'}(y|X, S) \ \text{or} \ P^t_k(y|X, S) \neq P^{t'}_{k'}(y|X, S)$, where $(t=t' \ \text{or} \ t \neq t')$.
\end{itemize}
\end{definition}

In other words, group-specific distributed concept drift occurs when different clients in a FL setting experience either distinct or similar group-specific concept drifts that can happen at the same or different timesteps and clients. These temporal and spacial dynamics can lead to challenges in maintaining fairness and accuracy over time.

It is important to highlight that the exploration and comprehensive understanding of group-specific distributed concept drift in the context of FL has not been introduced and explored in the existing literature. The unique challenges posed by this phenomenon, including the temporal misalignment and spatial variations in group-specific concept drift across different clients and timesteps, add an additional layer of complexity to the already intricate landscape of fairness in FL. As such, our work contributes to filling this gap by formally introducing and addressing the intricacies of group-specific distributed concept drift, shedding light on its potential impact on fairness within FL systems.

Figure \ref{fig:concept-drift} illustrates five different concept drift scenarios that can arise in FL with three distinct concepts (`A', `B' and `C'). 
\begin{itemize}
    \item Scenario 0: no concept drift - the underlying data distribution remains stable over time across all clients. There is no concept drift experienced by any client, and the system operates under static conditions.
    \item Scenario 1: general concept drift - all clients experience a uniform concept drift in the data distribution over time. The drift is consistent, affecting all clients equally, and reflects a global change in the overall data environment.
    \item Scenario 2: distributed concept drift - each client experiences independent changes in their local data distribution over time. The drifts are uncoordinated, meaning different clients may encounter different concepts, but these drifts happen independently across clients.
    \item Scenario 3: group-specific general concept drift - groups within the clients experience concept drift, but this drift happens uniformly across all clients for that group. For example, all clients experience the same concept drift for a specific sensitive group, though the drift remains general across the federation.
    \item Scenario 4: group-specific distributed concept drift - groups within different clients experience concept drift at different times. The group-specific concept drifts vary not only over time but also across different clients. This scenario introduces a unique complexity, as the concept drift occurs in a distributed manner, affecting different groups within different clients at different times.
\end{itemize}

We focus on studying scenario 4 - group-specific distributed concept drift, as it presents a unique and unexplored aspect of concept drift in the FL setting. This scenario demands novel techniques that consider the granularity of group-specific variations across different clients, which cannot be effectively handled by conventional methods aimed at distributed concept drift. By addressing this gap, our work introduces a framework that ensures fairness by detecting and adapting to these group-specific concept drifts.

\subsection{Objective}

We present the optimization objective employed to detect and address group-specific distributed concept drift while upholding the privacy of participating clients' data within the federation. We adopt a multiple global model approach where we aim to discover clusters of clients, each representing a distinct concept, and train each global model accordingly. To achieve this, we incorporate local group-specific loss monitoring on each client. In other words, for each group $s \in S$ we calculate its loss on each global model $\theta_m$, where $m$ is the index of each model, in the set of available global models, $GM$, and assign each client to a specific global model, accordingly.

This idea is inspired by previous work \cite{jothimurugesan2023federated}, which adopts a multiple global model approach based on global loss monitoring. However, the use of global loss monitoring in their approach poses limitations when dealing with group-specific concept drift. This is because global loss may not vary significantly over time in the presence of group-specific concept drift, especially when the dataset is imbalanced with respect to a sensitive group. In this case, changes in loss predominantly occur within specific groups. Consequently, relying solely on global loss may overlook variations introduced by group-specific concept drift.

We employ the following optimization objective for each global model $\theta_m$ in the set of global models, $GM$, at each timestep $t$:

\begin{equation}
\begin{split}
\min_{\theta_m \in GM} f(\theta_m)^t = \frac{\sum_{k \in K} G_k(\theta_m) \sum_{t' \leq t: w_k^{t'}=m} |D_k^{t'}|}{\sum_{k \in K}\sum_{t' \leq t: w_k^{t'}=m} |D_k^{t'}|} \\
w_k^{t} = \arg\min_m \sum_{s \in S} \ell(\theta_m)^{D_k^{t}}_{s} \\
\text{s.t.} \quad \forall \theta_m \in GM : \forall k: w_k^{t}=m : \forall s \in S : \\ \ell(\theta_{m})^{D_k^{t}}_{s} \leq \ell(\theta_{w_k^{t-1}})^{D_k^{t-1}}_{s} + \delta_s
\end{split}
\label{eq:objective}
\end{equation}
where: 
\begin{itemize}
    \item $t \in T$ denotes the timestep;
    \item $k \in K$ represents a client within the federation;
    \item $s \in S$ refers to a sensitive group;
    \item $\theta_m \in GM$ refers to a specific global model with index $m$;
    \item $w_k^{t}$ denotes the index of the global model that client $k$ has identified with at timestep $t$;
    \item $t' \leq t: w_k^{t'}=m$ is the set of timestamps where client $k$ has identified with global model $\theta_{m}$;
    \item $f(\theta_m)^t$ is the weighted sum of local objectives $G_k(\theta_m)$ of clients that identify with global model $\theta_m$;
    \item $\ell(\theta_m)^{D_k^{t}}_{s}$ is the loss for group $s$ and client $k$ on model $\theta_m$ at timestep $t$;
    \item $\delta_s$ represents a predefined threshold for the acceptable loss difference of group $s$ between two consecutive timestamps.
\end{itemize}

Hence, at each timestep, each client identifies with the model that yields the lowest sum of group losses, as long as the loss difference of each group $s$ between the current and the previous timestamp does not surpass the threshold for that group, $\delta_s$. If there is no such model, then a new global model will be created and added to the set of global models $GM$. Further details will be explained in the next section.

Finally, our framework involves training multiple global models, each assigned to a cluster of clients. Within each round, the training of each global model proceeds via standard FedAvg \cite{Federated_Learning_Avg} over the clients assigned to it. Since the aggregation weights for clients in a given cluster are fixed during that timestep, the training dynamics are equivalent to those in classical FedAvg. Therefore, existing convergence analysis for FedAvg \cite{li2020federated} applies directly within each cluster at each timestep.

\section{FairFedDrift: Fair Federated Learning under Group-Specific Distributed Concept Drift}\label{sec:fairfeddrift}

Our work offers a solution to address the challenge of group-specific distributed concept drift building upon an existing algorithm, FedDrift, presented in \cite{jothimurugesan2023federated}. Algorithm \ref{algo:fair-fed-drift} presents the pseudo-code of FairFedDrift\footnote{Source code can be found at: \url{https://github.com/teresalazar13/FairFedDrift}.}.

\begin{algorithm}[h!]
\caption{FairFedDrift}
\label{algo:fair-fed-drift}
 \begin{algorithmic}[1]
    \STATE Initialize global models GM: $\left[ \theta_0 \right]$, number of timesteps: $T$, number of rounds: $R$, number of clients: $K$, client identities $w_{k \in K}^{t=0}=0$, local epochs: $E$, local batch size: $|B|$, local learning rate: $\eta$, loss threshold: $\delta_s$.
    \FOR{each timestep $t \in T$}
        \FOR{each client $k \in K$}
            \IF{$\nexists \theta_m \in GM: \forall s \in S, \ell(\theta_m)^{D_k^t}_{s} \leq \ell(\theta_{w_k^{t-1}})^{D_k^{t - 1}}_{s} + \delta_s$}
                \STATE Add new $\theta$ to GM at index $i$
                \STATE $w_k^t = i, \quad \ell(\theta_{w_k^t})^{D_k^t}_{s} \gets \min_{j \in GM} \ell(\theta_j)^{D_k^t}_{s}$
            \ELSE{
                \STATE
                    $w_k^t = \arg\min_m$ \newline
                    \hspace*{2em}
                    $\sum_{s \in S}
                    \begin{cases}
                        \ell(\theta_m)^{D_k^t}_{s} 
                        \\ \quad \quad \text{if } \ell(\theta_m)^{D_k^t}_{s} \leq \ell(\theta_{w_k^{t-1}})^{D_k^{t-1}}_{s} + \delta_s \\
                        \infty \quad \text{else}
                    \end{cases}$
            }
            \ENDIF
        \ENDFOR

        \FOR{each $i, j$ in parallel from 1 to $|GM|$}
            \STATE $L_{ij} - L_{ii} \leftarrow \max_{k: \exists t' < t: w_k^{t'}=j,  l: \exists t* < t: w_l^{t*}=i} $ \newline 
                \hspace*{2em} 
                $\sum_{s \in S}
                \begin{cases}
                    \ell(\theta_i)^{\cup_{t'}D_k^{t'}}_{s} - \ell(\theta_i)^{\cup_{t*}D_l^{t*}}_{s} 
                    \\ \quad \quad \text{if } \ell(\theta_i)^{\cup_{t'}D_k^{t'}}_{s} - \ell(\theta_i)^{\cup_{t*}D_l^{t*}}_{s} \leq \delta_s \\
                    \infty \quad \text{else} \\
                \end{cases}$
            \STATE Compute $L_{ji} - L_{jj}$ similarly
            \STATE $Z_{ij}$ = $\max (L_{ij} - L_{ii}, L_{ji} - L_{jj}, 0)$
        \ENDFOR

        \WHILE{$\min Z_{ij} \neq \infty$}
            \STATE Add new model to $GM$: $\theta_k = \frac{\theta_i W_i + \theta_j W_j}{W_i + W_j}$, \newline 
            \hspace*{2em} $W_{n \in i, j} = \sum_{k \in K}\sum_{t' < t: w_k^{t'}=n}|D_{k}^{t'}|$ 
            \STATE $Z_{kl} = \max(Z_{il}, Z_{jl}, 0)$ for all $l$
            \STATE $w_k^{t' \leq t}$ = $w_i^{t' \leq t}$ + $w_j^{t' \leq t}$
            \STATE Delete $\theta_i$ and $\theta_j$
        \ENDWHILE

        \FOR{each round $r \in R$}
            \FOR{each client $k \in K$}
                \FOR{each global model $\theta_m \in GM$}
                    \STATE $\theta_{m, k}$ = $\theta_m$
                    \FOR{each local epoch $e \in E$}
                        \FOR{batch $b \in B$ from $\cup_{t' \leq t:w_{k}^{t'}=m} D_k^{t'} $}
                            \STATE $\theta_{m, k}$ = $\theta_{m,k} - \eta \nabla \ell(\theta_m, b)$
                        \ENDFOR
                    \ENDFOR
                \ENDFOR
            \ENDFOR
            \FOR{each global model $\theta_m \in GM$}
                \STATE $\theta_m$ = $\frac{ \sum_{k \in K} \theta_{m, k} \sum_{t' \leq t: w_{k}^{t'}=m} |D_k^{t'}|}{\sum_{k \in K} \sum_{t' \leq t: w_{k}^{t'}=m} |D_k^{t'}| }  $ 
            \ENDFOR
        \ENDFOR
     \ENDFOR
 \end{algorithmic}
\end{algorithm}

%\subsection{Algorithmic Design}

\subsubsection{Multi-model approach} The core idea behind this algorithm is to recognize that a single global model is ill-suited for handling distributed group-specific concept drifts, leading to the adoption of a multi-model strategy. In the beginning of the federation, there is a single global model, $\theta_0$, and all clients are associated with it, $w_{k \in K}^{t=0}=0$. At each timestep, $t$, each client $k$ has new incoming data $D_k^t$ and associates itself with the model that yields the lowest sum of losses of every group $s \in S$, as long as the difference in loss of each group $s$ between the current and the previous timestep does not surpass the threshold for that group, $\delta_s$. In contrast to FairFedDrift, which is designed with fairness considerations, FedDrift is fairness-unaware and associates itself with the model that yields the lowest global loss, without accounting for group-specific losses.

\subsubsection{Group-specific distributed concept drift detection} A drift is identified if there is no global model in the set of global models where for any group $s \in S$ the difference in losses of two consecutive timestamps of that group is smaller than the threshold for that group, $\delta_s$ (line 4). In this case, a new global model is created. Contrary to FairFedDrift, in FedDrift a drift is identified if the difference between the lowest global loss (not group-specific) of two consecutive timestamps is greater than a threshold, $\delta$.

\subsubsection{Global model merging} Importantly, since new concepts might simultaneously emerge at different clients, a merge operation is employed to consolidate models that correspond to the same concept. This operation involves the creation of a distance matrix (lines 11-15), $Z$, where each cell is calculated for each pair of global models, $i$ and $j$. The distance $Z_{ij}$ corresponds to the maximum value of $L_{ij} - L_{ii}$, $L_{ji} - L_{jj}$, and 0, where $L_{ij} - L_{ii}$ represents the sum of the group loss differences of model $\theta_i$ when tested on data associated with model $j$ and $i$. For a model $j$, there can be multiple clients $k$ that have identified with it: $k: \exists t' < t: w_k^{t'}=j$. Here, $\cup_{t'}D_k^{t'}$ represents the data of client $k$ associated with model $j$. If the loss difference of any group $s$ exceeds the threshold for that group, $\delta_s$, the value $\infty$ is assigned to $Z_{ij}$. The merging process (lines 16-21) combines the two models with the lowest distances in the matrix (provided they do not reach $\infty$) and unifies cluster identities by averaging their models, with the weighting based on the size of each model's training dataset. The distance between clusters of multiple elements is measured as the maximum distance between their constituent parts. In contrast to FairFedDrift, FedDrift only focuses on the global loss for merging global models, without considering group-specific losses.

\subsubsection{Training and Model Averaging} The training and model averaging process (lines 22-36) involves multiple communication rounds, denoted by $R$. In each round, every global model $\theta_m$ in $GM$ is updated based on the contributions from clients associated with that model. The update is performed through iterative optimization, where each client $k$ trains its local model $\theta_{m, k}$ if it had been associated with model $m$ until that timestep. These local models are then aggregated to update the global model $\theta_m$ by considering the weighted sum of individual models, where the weights are proportional to the sizes of the respective client datasets.

To sum up, the choice of considering the loss of each group $s \in S$ in FairFedDrift, in contrast to FedDrift \cite{jothimurugesan2023federated}, stems from its effectiveness in handling group-specific concept drift as it inherently better suited to capture the nuances of fairness in these scenarios where specific groups may experience concept drift at different timesteps. This is attributed to the fact that, in instances of group-specific concept drift, a global loss metric may not manifest substantial variations. FairFedDrift addresses this gap through three core modifications: (1) during model selection, it uses the sum of group-specific losses rather than global loss to ensure that all groups are considered in the selected model; (2) during drift detection, it monitors each group's loss individually, flagging drift when any group-specific variation exceeds the threshold; and (3) in model merging, it imposes per-group thresholds to merge models together. These changes are essential for maintaining fairness over time in dynamic, heterogeneous environments. We provide a detailed comparison between FairFedDrift and baseline methods, and empirically validate its advantages in Section \ref{sec:experiments}.

\subsubsection{Window Size} Finally, instead of retaining the entire history of client identities $\cup_{t' \leq t} w_k^{t'}$ and local datasets $\cup_{t' \leq t} D_k^{t'}$ at each timestamp, clients could opt to maintain only a sliding window of the most recent timesteps \cite{jothimurugesan2023federated}. The window size parameter determines how much past information is used during model assignment and concept drift detection. Mathematically, this adjustment limits the history to a window of size $w$, encompassing timestamps in the range $\max(0, t - w) \leq t' \leq t$. Specifically, for each timestep $t$, the history of client identities $w_k^{t'}$ and local datasets $D_k^{t'}$ are restricted to: $\cup_{\max(0, t - w) \leq t' \leq t} w_k^{t'}$ and $\cup_{\max(0, t - w) \leq t' \leq t} D_k^{t'}$, respectively. This change reduces the computational burden while potentially affecting fairness and accuracy, as less historical data is considered.

\section{Methodology and Experimental Design}\label{sec:methodology}

In our experimental setup, we operate within a FL framework that operates continuously over time involving K clients. At each timestep, every client receives a new batch of training data. After completing training for a given timestep, which can span multiple communication rounds, for all clients, we test the global model associated with that client over the batch of data arriving at that client at the following timestep (between lines 3 and 4, in Algorithm \ref{algo:fair-fed-drift}).

\begin{figure*}
    \centering
    \includegraphics[width=\linewidth]{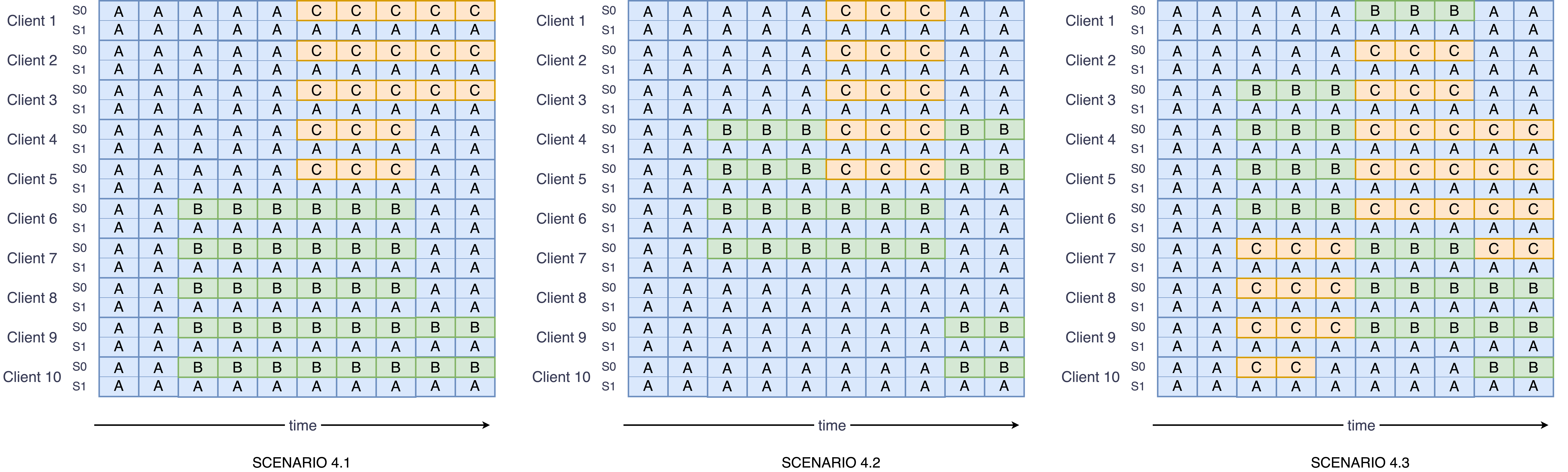}
    \caption{Test scenarios of group-specific distributed concept drift.}
    \label{fig:scenarios}
\end{figure*}

\subsection{Datasets and Group-Specific Distributed Concept Drift Simulation}

In our experimental setup, we leverage diverse datasets derived from established benchmarks in machine learning: MNIST \cite{mnist}, FEMNIST \cite{fashion-mnist}, CIFAR-10 \cite{cifar}, and Adult \cite{adult}. These datasets were chosen due to their extensive size, making them well-suited for this FL setup, where each client necessitates a substantial number of data instances across different timesteps. In addition, these datasets have also been employed in fairness studies \cite{fairness-studies, fairness-studies-2}, which makes them suitable to our experiments. To simulate this FL scenario, we partition each dataset into $K$ clients and $T$ timesteps and introduce group-specific concept drift within the datasets at specific timesteps with distinct concepts (`A', `B', `C', `D', `E').

\subsubsection{\texttt{MNIST-GDrift}} The MNIST dataset consists of gray-scale handwritten digit images, commonly used for digit recognition tasks \cite{mnist}. We introduce one sensitive attribute, $S$, with two groups ($S=1$ and $S=0$), representing distinct image characteristics. For $S=0$ images, we invert the background and digit colors and rotate 90 degrees in comparison to standard MNIST images ($S=1$). To simulate four distinctive group-specific concepts (`B', `C', `D', `E') from the original distribution (`A'), we label-swap images for $S=0$: for drift `B', we swap labels `1' and `2'; for drift `C', we swap labels `2' and `3'; for drift `D', we swap labels `3' and `4'; for drift `E', we swap labels `4' and `5'.

\subsubsection{\texttt{FEMNIST-GDrift}} The FEMNIST is a dataset featuring gray-scale images of 10 types of clothing items, designed as an alternative to MNIST for benchmarking image classification algorithms \cite{fashion-mnist}. To create different groups and simulate group-specific concept drift, we use the same strategies as the \texttt{MNIST-GDrift} dataset.

\subsubsection{\texttt{CIFAR-10-GDrift}} The CIFAR-10 dataset is an established dataset used for object recognition. It is a subset of the 80 million tiny images dataset and consists of color images containing one of 10 object classes. We introduce one sensitive attribute, $S$, with two groups ($S=1$ and $S=0$). Images where $S=0$ correspond to classes 0-4, while images with $S=1$ correspond to classes 5-9. To simulate four distinctive group-specific concepts (`B', `C', `D', `E') from the original distribution (`A'), we label-swap images for $S=0$: for drift `B', we swap labels `0' and `1'; for drift `C', we swap labels `1' and `2'; for drift `D', we swap labels `2' and `3'; for drift `E', we swap labels `3' and `4'.

\subsubsection{\texttt{Adult-GDrift}} We experiment on the Adult dataset \cite{adult} commonly used in fair machine learning research. The prediction task for the Adult dataset involves determining whether a person's income exceeds \$50K/yr. In this context, race serves as the sensitive attribute, where $S=1$ typically corresponds to White individuals, and $S=0$ represents non-White individuals. To introduce concept drift, we modify features of the samples where $S=0$ and change the outcome of each of 4 distinct groups created based on gender (`Male' and `Female') and occupation (`Professional' - `Adm-clerical', `Exec-managerial', `Prof-specialty', `Tech-support'; and `Technical' - 'Armed-Forces', `Craft-repair', `Machine-op-inspct', `Farming-fishing', `Handlers-cleaners', `Other-service', `Priv-house-serv', `Protective-serv', `Sales', `Transport-moving'). For concept `A', the dataset remains unchanged. For concept `B', the target class of non-White individuals who are Male and have `Professional' occupations is set to 1. For concept `C', the target class of non-White individuals who are Female and have `Professional' occupations is set to 1. For concept `D', the target class of non-White individuals who are Male and have `Technical' occupations is set to 1. For concept `E', the target class of non-White individuals who are Female and have `Technical' occupations is set to 1. 

\subsubsection{$\alpha$ parameter} The parameter $\alpha$ is an important element in our experimental design, determining the size of the unprivileged group relative to the privileged group. The size of the unprivileged group is a result of multiplying the size of the privileged group by $\alpha$. By adjusting the value of $\alpha$, we can control the balance between the groups in our datasets, allowing us to explore different scenarios where one group may be underrepresented or over-represented relative to the other. As such, when $\alpha$ is closer to zero, the relative size of the unprivileged group ($S=0$) is smaller. To generate a specific number of instances depending on the value of $\alpha$ for the \texttt{MNIST-GDrift} and \texttt{FEMNIST-GDrift} datasets, we adjust the number of transformed images accordingly. For the \texttt{CIFAR-10-GDrift} dataset, we reduce the number of images from the unprivileged group ($S=0$, classes 0-4) by sampling the dataset based on the specified $\alpha$ value, while keeping all images for the privileged group ($S=1$, classes 5-9). Finally, for the \texttt{Adult-GDrift} dataset, the process is more intricate as the original dataset is oversampled, whose default representation is equivalent to $\alpha=0.16$. If $\alpha > 0.16$, we generate unprivileged instances, and if $\alpha < 0.16$, we generate privileged instances. We employ the SMOTE \cite{smote} algorithm for creating these instances.

\subsubsection{Scenarios} We design five distinct scenarios of group-specific distributed concept drift with distinct concepts (`A', `B', `C', `D', `E'), as illustrated in Figure \ref{fig:scenarios}. These aim to provide a comprehensive exploration of the temporal and spatial dynamics that can impact fairness and accuracy in FL settings. More specifically, each scenario is constructed to vary in complexity based on three key factors:
\begin{itemize}
    \item Number of distinct concepts – The number of different concepts (i.e. ‘A’, ‘B’, ‘C’, ‘D’ and ‘E’) appearing in the scenario.
    \item Number of concept drifts – The number of times a concept drift occurs across all clients.
    \item Number of timesteps with drift – The number of timesteps where at least one client experiences a drift.
\end{itemize}

Below, we provide a more explicit distinction between the scenarios:
\begin{itemize}
    \item Scenario 4.1 – Involves three distinct concepts appearing sequentially, with 15 drift events across three timesteps.
    \item Scenario 4.2 – Also has three concepts, but more recurring drifts occur, leading to 18 drift events across the same three timesteps.
    \item Scenario 4.3 – Introduces simultaneous concept appearances, increasing complexity to 23 drift events across four timesteps.
    \item Scenario 4.4 – Expands to five distinct concepts, resulting in 30 drift events across six timesteps, with more diverse transitions.
    \item Scenario 4.5 – The most complex setting, featuring five concepts and a highly unpredictable pattern, with 38 drift events across seven timesteps.
\end{itemize}

By evaluating performance across this spectrum of drift scenarios, we ensure that our approach is tested under both controlled and highly dynamic drift conditions, reflecting the diverse challenges found in real-world FL applications.

\subsection{Implementation Details}

We employ a convolutional neural network (CNN) for the \texttt{MNIST-GDrift} and \texttt{FEMNIST-GDrift} datasets featuring a convolutional layer, a max-pooling layer, a fully connected layer, and a softmax activation function. For the \texttt{CIFAR-10-GDrift} dataset we use a ShuffleNet v2 \cite{ma2018shufflenet} model pre-trained on ImageNet, and extend it by adding a fully connected layer with 256 units, followed by ReLU, dropout, batch normalization, another fully connected layer, using softmax for multi-class classification. For the \texttt{Adult-GDrift} dataset we employ a fully connected feed-forward neural network with one hidden layer with 10 $tanh$ units, and a sigmoid output neuron (binary classification). As explained before, each test scenario is constructed to span $T = 10$ timesteps, with $K = 10$ clients participating. At each timestep, clients perform $R = 10$ local communication rounds using a learning rate of $\eta = 0.1$. We experiment with different values of $\delta_s$: $\{0.05, 0.1, 0.25, 0.5, 0.75, 1.0, 1.25, 1.5, 1.75, 2.0\}$. Furthermore, for the \texttt{MNIST-GDrift} and \texttt{FEMNIST-GDrift} datasets we set $|B|=32$ and $E=5$. For the \texttt{CIFAR-10-GDrift} dataset we set $|B|=128$ and $E=30$. For the \texttt{Adult-GDrift} dataset we set $|B|=10$ and $E=10$.

\subsection{Metrics}

We evaluate the algorithms based on both accuracy (ACC) and fairness (AEQ, OEQ, OPP).

\begin{definition}[Accuracy Equality (AEQ)]
Accuracy Equality (AEQ) \cite{berk2021fairness} is a measure of fairness that calculates the ratio of accuracies across two groups ($S=0$ and $S=1$):
\begin{equation}
    \text{AEQ} = \frac{P [\hat{Y} = Y \mid S = 0]}{P [\hat{Y} = Y \mid S = 1]}
\end{equation}
\end{definition}

\begin{definition}[Overall Equality of Opportunity (OEQ)]
Overall Equality of Opportunity (OEQ) \cite{equality-of-opportunity} compares the true positive rates (TPRs) for two groups. Let $TPR_0 = P [\hat{Y} = c \mid S = 0, Y = c]$ and $TPR_1 = P [\hat{Y} = c \mid S = 1, Y = c]$ represent the TPRs for group $S=0$ and $S=1$ for the $c^{th}$ target class, respectively. The formula for OEQ is given by:

\begin{equation}
\text{OEQ} = 
\left\{
\begin{array}{ll} 
\frac{1}{C} \sum_{c=1}^{C} \frac{TPR_0}{TPR_1} & \text{If} \quad Y(S=0) = Y(S=1) \vspace{0.3cm} \\
\frac{ \frac{1}{C_0} \sum_{c=i}^{C_0} TPR_0}{\frac{1}{C_1} \sum_{c=j}^{C_1} TPR_1} & \text{If} \quad Y(S=0) \neq Y(S=1) \vspace{0.3cm}
\end{array}
\right.
\end{equation}
where $C$ is the number of target classes. In cases where the target classes for both groups are overlapping (i.e. $Y(S=0) = Y(S=1)$), the metric compares the TPRs across the groups for each class. However, if the classes are non-overlapping (i.e. $Y(S=0) \neq Y(S=1)$), as in the \texttt{CIFAR-10-GDrift} dataset, the metric evaluates the TPRs for each class separately within each group and then computes the ratio of the average TPRs across groups. Here, \( C_0 \) is the number of classes \( c \) where the count of instances with \( S=0 \) and \( Y = c \) is nonzero (where $TPR_0$ can be calculated), and \( C_1 \) is defined similarly for \( S=1 \).

\end{definition}

\begin{definition}[Overall Predictive Parity (OPP)]
Overall Predictive Parity (OPP) \cite{chouldechova2017fair} compares the positive predictive values (PPVs) for two groups. Let $PPV_0 = P [Y = c \mid S = 0, \hat{Y} = c]$ and $PPV_1 = P [Y = c \mid S = 1, \hat{Y} = c]$ represent the PPVs for group $S=0$ and $S=1$ for the $c^{th}$ target class, respectively.  The formula for OPP is given by:
\begin{equation}
\text{OPP} = 
\left\{
\begin{array}{ll} 
\frac{1}{C} \sum_{c=1}^{C} \frac{PPV_0}{PPV_1} & \text{If} \quad Y(S=0) = Y(S=1) \vspace{0.3cm} \\
\frac{ \frac{1}{C_0} \sum_{c=i}^{C_0} PPV_0}{\frac{1}{C_1} \sum_{c=j}^{C_1} PPV_1} & \text{If} \quad Y(S=0) \neq Y(S=1) \vspace{0.3cm}
\end{array}
\right.
\end{equation}
where $C$ is the number of target classes. In a similar way to the OEQ metric, \( C_0 \) is the number of classes \( c \) where the count of instances with \( S=0 \) and \( \hat{Y} = c \) is nonzero (where $PPV_0$ can be calculated), and \( C_1 \) is defined similarly for \( S=1 \).
\end{definition}

In these metrics (AEQ, OEQ, OPP), a value of 1 indicates ideal fairness, and the numerator corresponds to the group with the lowest value to account for cases of reverse discrimination.

\subsection{Baselines}

FairFedDrift is compared against three pivotal algorithms: FedAvg, FedDrift \cite{jothimurugesan2023federated} and Oracle. The Oracle algorithm acts as an idealized benchmark, assuming access to concept IDs during training and using multiple-model training based on ground-truth clustering. While unrealistic in real-world scenarios, where concept IDs are unknown, Oracle provides an upper-bound performance reference to measure the effectiveness of practical approaches \cite{jothimurugesan2023federated}. Given the unique problem of group-specific distributed concept drift, traditional baselines for fair FL are absent, as explained in Section \ref{sec:related-work}. We found that these achieve very similar results to FedAvg on these datasets under group-specific distributed concept drift, as they lack the adaptability to effectively address the challenges posed by these evolving data distributions. Nevertheless, Oracle serves as a perfect algorithm for this problem setup and should act as an ideal baseline.

\section{Experimental Results}\label{sec:experiments}

In this section, we unveil the outcomes of our experiments, addressing a distinctive challenge in our investigation - that of group-specific distributed concept drift. It is essential to underscore that our contribution extends beyond the mere presentation of algorithms' results; we introduce and explore this novel problem, shedding light on the intricate challenges it presents.

\subsection{Results Across Multiple Datasets and Scenarios}

\begin{table}[h!]
\centering
\renewcommand{\arraystretch}{1.2}
\setlength{\tabcolsep}{4pt}
\begin{tabular}{cccccc:c}
    \toprule
    \multirow{2}{*}{Dataset} & \multirow{2}{*}{$\alpha$} & \multirow{2}{*}{Metric} & \multicolumn{4}{c}{Algorithm} \\
    \cline{4-7}
    & & & FedAvg & FedDrift & FairFedDrift & Oracle \\ 
    \midrule
    \multirow{8}{*}{\rotatebox[origin=c]{90}{\parbox{2cm}{\centering \texttt{MNIST}\\ \texttt{GDrift}}}}
    & \multirow{4}{*}{0.1} 
      & AEQ & 0.83+-0.09 & 0.84+-0.09 & \textbf{0.91+-0.08} & 0.93+-0.05 \\ 
    & & OEQ & 0.80+-0.07 & 0.81+-0.08 & \textbf{0.87+-0.09} & 0.89+-0.05 \\ 
    & & OPP & 0.79+-0.08 & 0.81+-0.08 & \textbf{0.87+-0.08} & 0.89+-0.05 \\  
    & & ACC & 0.95+-0.01 & \textbf{0.96+-0.02} & 0.94+-0.04 & 0.95+-0.03 \\
    \cline{2-7}
    & \multirow{4}{*}{0.05} 
      & AEQ & 0.80+-0.11 & 0.82+-0.10 & \textbf{0.85+-0.12} & 0.89+-0.09 \\ 
    & & OEQ & 0.76+-0.09 & 0.79+-0.09 & \textbf{0.80+-0.12} & 0.83+-0.09 \\ 
    & & OPP & 0.75+-0.09 & 0.79+-0.09 & \textbf{0.80+-0.12} & 0.83+-0.09 \\  
    & & ACC & 0.95+-0.02 & \textbf{0.96+-0.02} & 0.94+-0.03 & 0.95+-0.02 \\
    \midrule
    \multirow{8}{*}{\rotatebox[origin=c]{90}{\parbox{2cm}{\centering \texttt{FEMNIST}\\ \texttt{GDrift}}}}
    & \multirow{4}{*}{0.1} 
    & AEQ   & 0.82+-0.09 & 0.82+-0.09 & \textbf{0.90+-0.07} & 0.91+-0.05 \\
    & & OEQ & 0.75+-0.07 & 0.76+-0.07 & \textbf{0.81+-0.07} & 0.83+-0.05 \\
    & & OPP & 0.75+-0.08 & 0.76+-0.08 & \textbf{0.82+-0.06} & 0.83+-0.05 \\
    & & ACC & \textbf{0.87+-0.02} & \textbf{0.87+-0.04} & 0.85+-0.04 & 0.86+-0.03 \\
    \cline{2-7}
    & \multirow{4}{*}{0.05} 
    & AEQ   & 0.78+-0.11 & 0.78+-0.11 & \textbf{0.83+-0.11} & 0.85+-0.08 \\
    & & OEQ & 0.68+-0.09 & 0.70+-0.08 & \textbf{0.72+-0.09} & 0.73+-0.08 \\
    & & OPP & 0.68+-0.10 & 0.69+-0.09 & \textbf{0.73+-0.09} & 0.74+-0.09 \\
    & & ACC & 0.87+-0.02 & \textbf{0.88+-0.02} & 0.86+-0.04 & 0.87+-0.03 \\
    \midrule
    \multirow{8}{*}{\rotatebox[origin=c]{90}{\parbox{2cm}{\centering \texttt{CIFAR-10}\\ \texttt{GDrift}}}}
    & \multirow{4}{*}{0.5} 
    & AEQ   & 0.54+-0.17 & 0.71+-0.14 & \textbf{0.73+-0.12} & 0.77+-0.09 \\
    & & OEQ & 0.54+-0.17 & 0.71+-0.14 & \textbf{0.73+-0.12} & 0.76+-0.09 \\
    & & OPP & 0.64+-0.18 & 0.83+-0.14 & \textbf{0.85+-0.12} & 0.89+-0.08 \\
    & & ACC & \textbf{0.76+-0.04} & 0.74+-0.10 & \textbf{0.76+-0.09} & 0.81+-0.04 \\
    \cline{2-7}
    & \multirow{4}{*}{0.25} 
    & AEQ   & 0.41+-0.15 & 0.51+-0.14 & \textbf{0.55+-0.13} & 0.62+-0.09 \\
    & & OEQ & 0.41+-0.15 & 0.51+-0.14 & \textbf{0.55+-0.12} & 0.62+-0.09 \\
    & & OPP & 0.57+-0.19 & 0.70+-0.18 & \textbf{0.78+-0.15} & 0.83+-0.10 \\
    & & ACC & \textbf{0.81+-0.03} & 0.77+-0.10 & 0.78+-0.08 & 0.83+-0.03 \\
    \midrule
    \multirow{8}{*}{\rotatebox[origin=c]{90}{\parbox{2cm}{\centering \texttt{Adult}\\ \texttt{GDrift}}}}
    & \multirow{4}{*}{0.1} 
    & AEQ   & 0.92+-0.07 & 0.93+-0.05 & \textbf{0.95+-0.04} & 0.94+-0.04 \\
    & & OEQ & 0.83+-0.10 & 0.83+-0.11 & \textbf{0.85+-0.09} & 0.86+-0.09 \\
    & & OPP & 0.86+-0.09 & 0.88+-0.08 & \textbf{0.89+-0.07} & 0.89+-0.07 \\
    & & ACC & 0.83+-0.02 & \textbf{0.84+-0.02} & 0.83+-0.02 & 0.83+-0.02 \\
    \cline{2-7}
    & \multirow{4}{*}{0.05} 
    & AEQ   & 0.92+-0.07 & 0.93+-0.05 & \textbf{0.94+-0.04} & 0.94+-0.04 \\
    & & OEQ & 0.83+-0.10 & 0.82+-0.10 & \textbf{0.85+-0.08} & 0.86+-0.08 \\
    & & OPP & 0.86+-0.08 & 0.88+-0.07 & \textbf{0.89+-0.07} & 0.89+-0.06 \\
    & & ACC & \textbf{0.84+-0.01} & \textbf{0.84+-0.01} & \textbf{0.84+-0.01} & 0.84+-0.01 \\
    \bottomrule
\end{tabular}
\caption{Average results for the \texttt{MNIST-GDrift}, \texttt{FEMNIST-GDrift}, and \texttt{Adult-GDrift} datasets under different $\alpha$ values.}
\label{table:res}
\end{table}

We delve into the evaluation of the algorithms on the described datasets, considering different $\alpha$ values and distinct scenarios of group-specific distributed concept drift presented before. Table \ref{table:res} provides the average results of all clients on all timesteps. The results demonstrate the efficacy of FairFedDrift in achieving favorable fairness outcomes (AEQ, OEQ, OPP) with minimal impact on accuracy. Notably, FairFedDrift's fairness results results closely approaches that of the idealized Oracle algorithm. 

The following sections will elucidate the mechanisms behind FairFedDrift's superior performance and provide an in-depth analysis of why it achieves better fairness outcomes.

\subsection{Unveiling the Influence of \(\alpha\) on Loss}

We explore how varying the $\alpha$ parameter, which determines the size of the unprivileged group relative to the privileged group, impacts both the overall loss $\ell$ and group-specific loss $\ell_{S=0}$. We use FedAvg as the base algorithm to observe the behavior of the model under two different $\alpha$ values: 0.1 and 0.5. These values represent different levels of heterogeneity in group distributions, with lower $\alpha$ values indicating more heterogeneous groups and higher $\alpha$ values representing more homogeneous ones.

\begin{figure}[h!]
    \centering
    \includegraphics[width= \columnwidth]{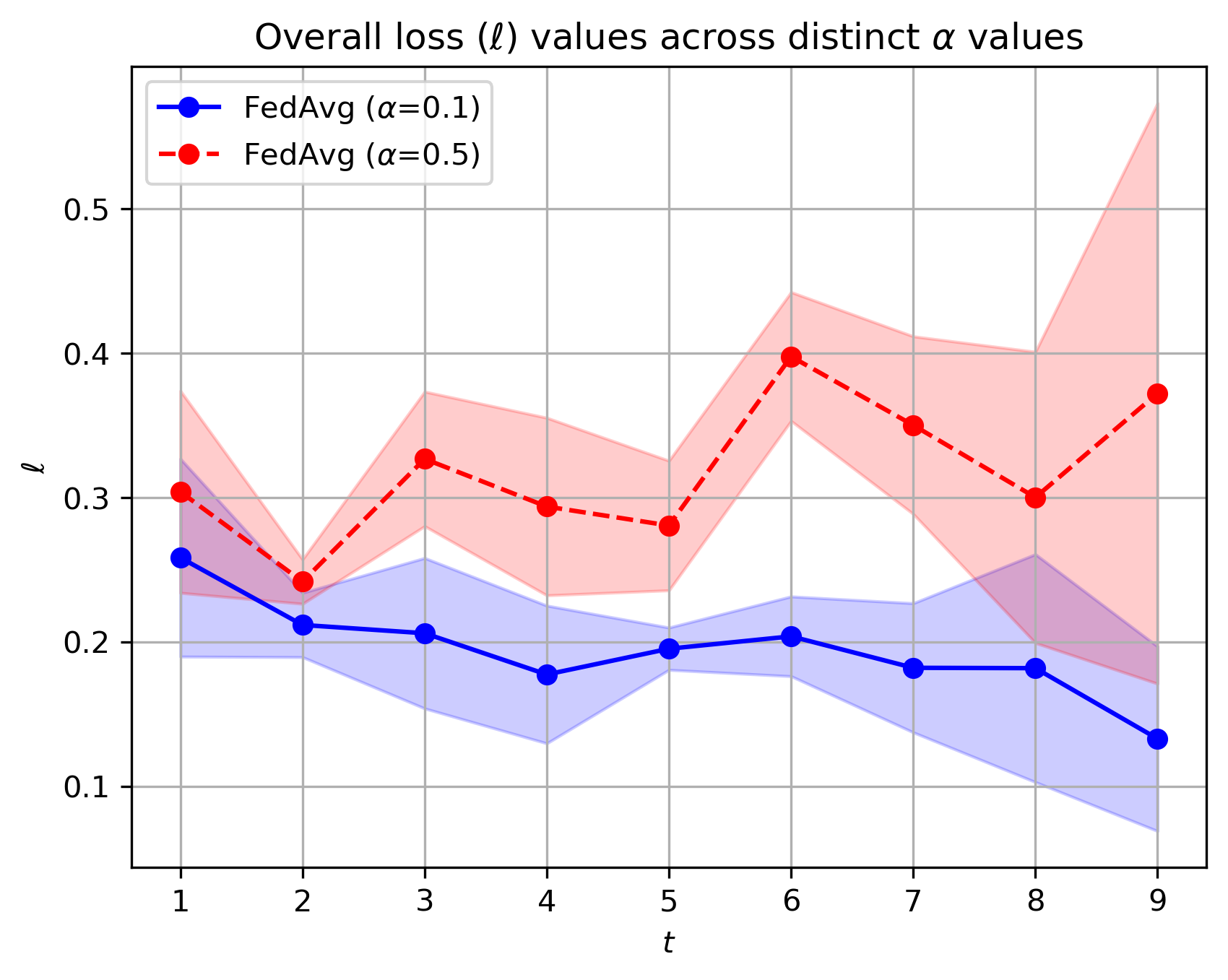} \\
    \includegraphics[width= \columnwidth]{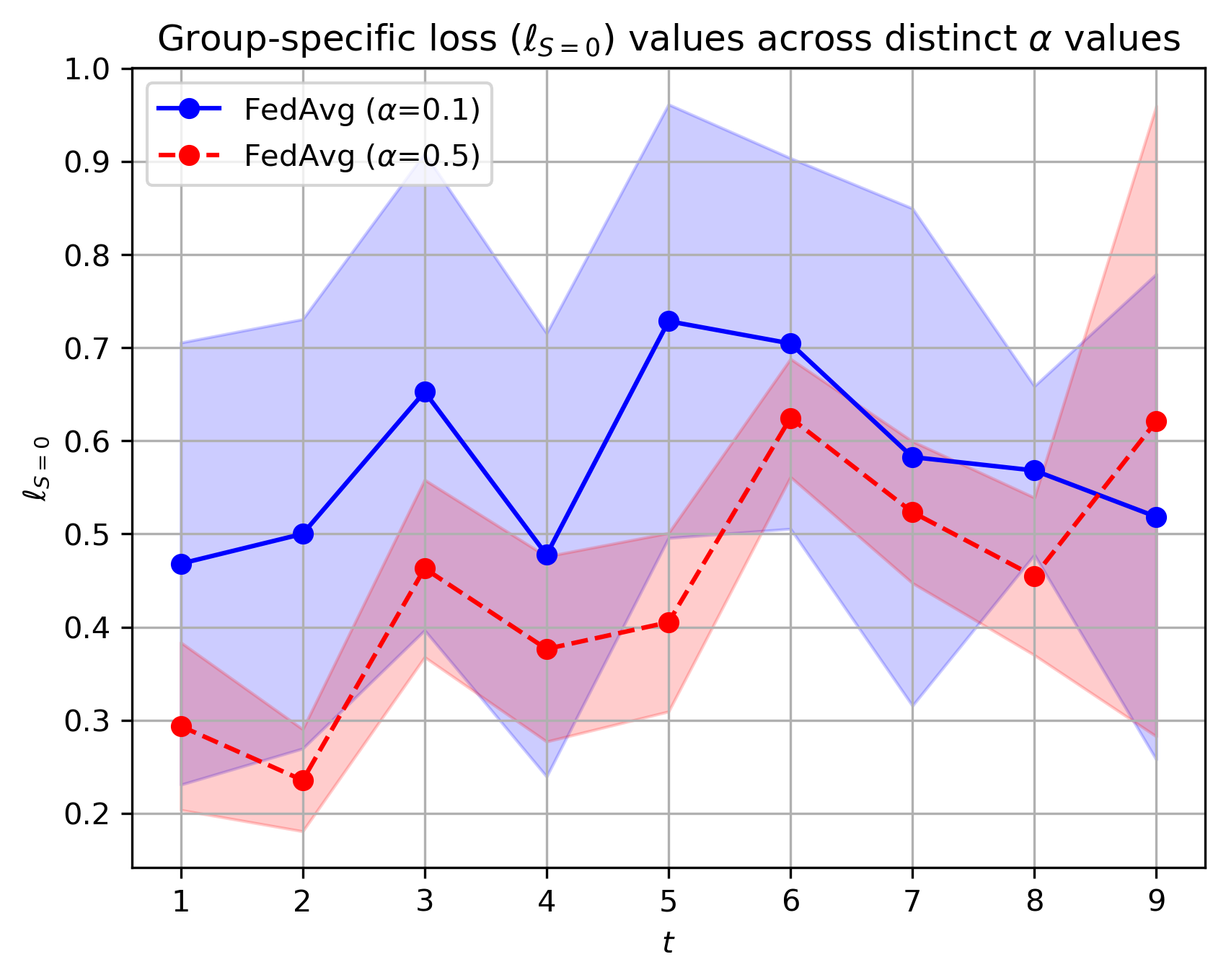}
    \caption{Average clients' loss ($\ell$) and group-specific ($S=0$) loss ($\ell_{S=0}$) results for the \texttt{MNIST-GDrift} dataset considering Scenario 4.1, $\alpha=0.1$ and $\alpha=0.5$.}
    \label{fig:loss-results}
\end{figure}

\begin{figure*}[h]
    \centering
    \includegraphics[width=\linewidth]{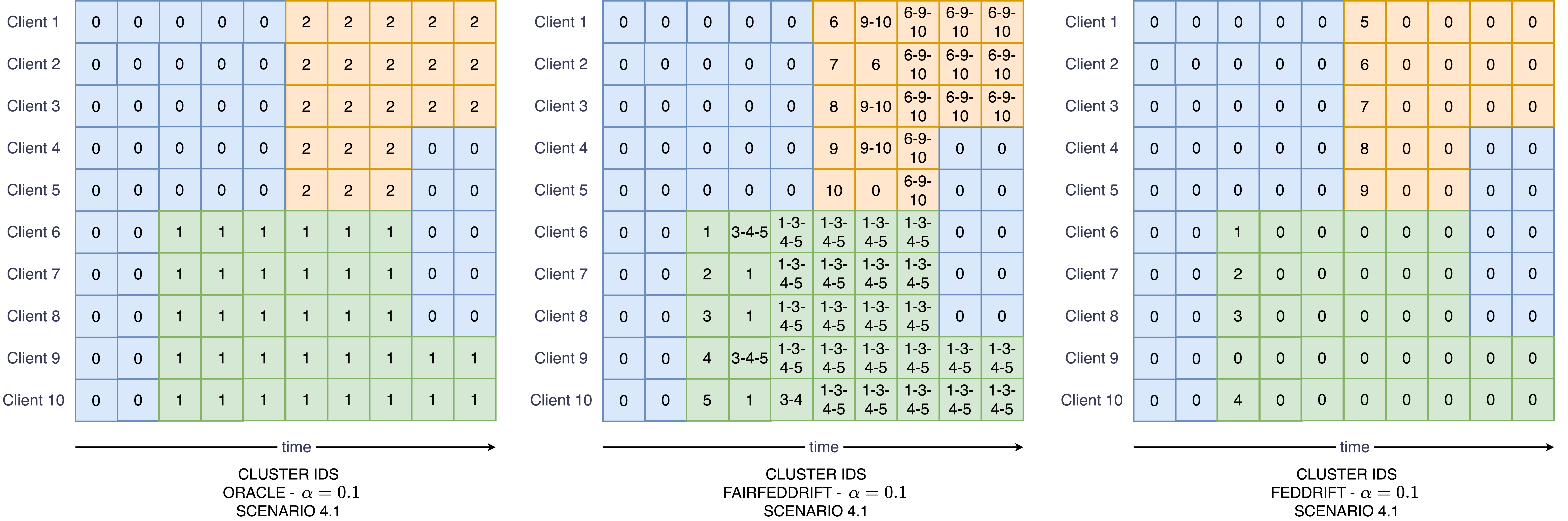}
    \caption{Clustering of FedDrift, FairFedDrift and Oracle on the \texttt{MNIST-GDrift} dataset, considering \(\alpha=0.1\) and Scenario 4.1. Each cell represents the model ID at each client and timestep, while background colors denote the ground-truth concept.}
    \label{fig:clusters}
\end{figure*}

Figure \ref{fig:loss-results} shows the results of these experiments. For $\alpha=0.1$, the overall loss $\ell$ remains relatively stable over time, fluctuating only slightly. However, when we isolate the loss for the unprivileged group ($S=0$), we observe much more substantial variation. This indicates that even though the model's overall performance does not appear to degrade, certain subsets of the population (in this case, the unprivileged group) are experiencing higher loss, which raises fairness concerns. This scenario highlights a significant challenge: relying on overall performance as the sole indicator of drift can obscure issues faced by smaller or less represented groups.

FedDrift, designed to detect concept drift by monitoring loss changes, faces challenges in these situations. Since the overall loss remains stable, the algorithm may struggle to detect drift in specific subgroups such as the unprivileged group. Consequently, tuning the threshold $\delta$ in FedDrift becomes difficult - smaller values can lead to too many clusters, while larger values may fail to detect group-specific concept drift. This illustrates the limitations of FedDrift in addressing subtle changes affecting only a portion of the population, as previously demonstrated.

Conversely, for $\alpha=0.5$, both the overall loss and the unprivileged group's loss exhibit more pronounced variations. In this case, the unprivileged group is better represented among the clients, leading to more noticeable loss changes. The fluctuations in both the overall model loss and the group-specific loss occur together, making it easier for FedDrift to detect drift and adjust accordingly. This scenario reflects a situation where drift detection is more effective, as group-specific loss changes significantly impact the overall performance.

The key takeaway from these experiments is that in scenarios with highly heterogeneous distributions, where the unprivileged group is less represented, concept drift can be difficult to detect using traditional drift detection methods that do not account for group-specific variations. In real-world applications, heterogeneous distributions are quite common, and unprivileged groups are typically in the minority. As such, detecting drift in these groups is important for maintaining fairness across the population. Models that rely solely on overall performance metrics risk overlooking fairness issues and disproportionately affecting underrepresented groups. This underscores the importance of designing drift detection algorithms, such as the one present in the FairFedDrift algorithm, that are sensitive to group-specific performance changes, even when overall metrics appear stable.

\begin{figure}[h]
    \centering
    \includegraphics[width=\linewidth]{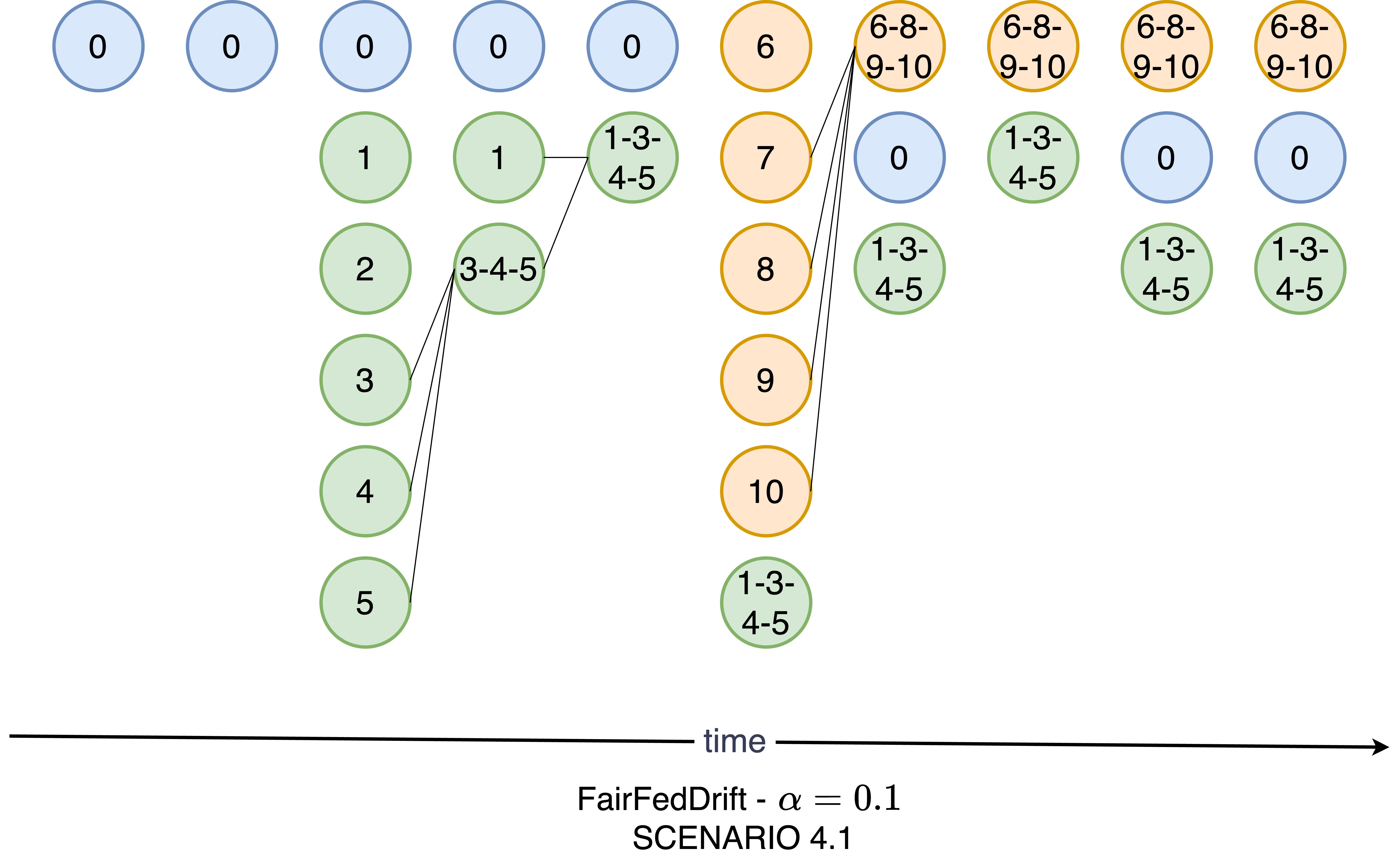}
    \caption{Model merge process over time in FairFedDrift on the \texttt{MNIST-GDrift} dataset, considering \(\alpha=0.1\) and Scenario 4.1. Each node represents a global model at a specific timestep, with background colors denoting the original concepts in which the models where created. Edges indicate model merges based on group-specific loss thresholds. This Figure complements Figure \ref{fig:clusters}, providing an aggregated view of how model clusters evolve and merge over time.}
    \label{fig:merging}
\end{figure}

\begin{figure*}[h!]
    \centering
    \includegraphics[width=\linewidth]{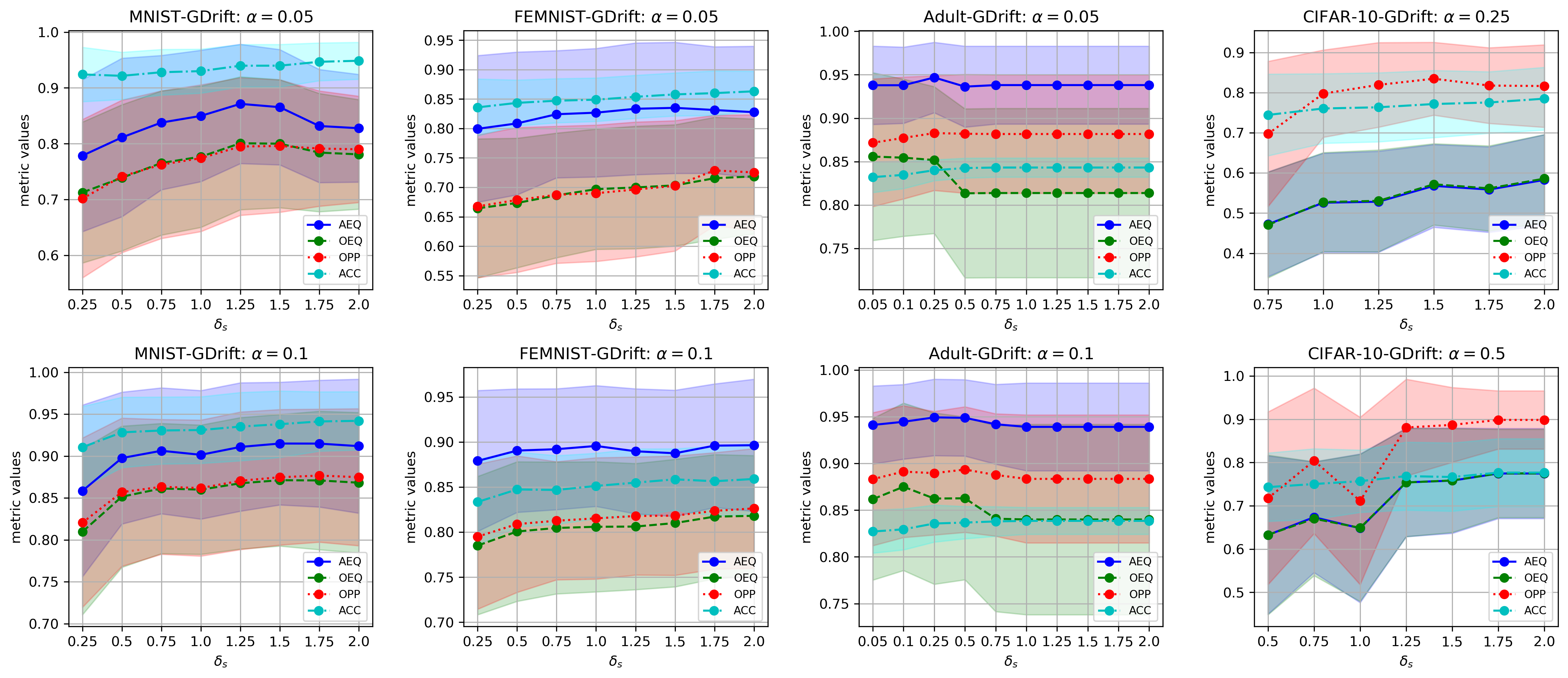}
    \caption{Effect of $\delta_s$ on fairness and performance across the \texttt{MNIST-GDrift}, \texttt{FEMNIST-GDrift}, \texttt{Adult-GDrift} and \texttt{CIFAR-10-GDrift} datasets evaluated under Scenario 4.1.}
    \label{fig:delta_s}
\end{figure*}

\subsection{Comparative Cluster Creation: FedDrift vs. FairFedDrift}

Figure \ref{fig:clusters} presents the clustering outcomes achieved by FedDrift, FairFedDrift and Oracle on the \texttt{MNIST-GDrift} dataset considering $\alpha=0.1$ and Scenario 4.1. Although both FedDrift and FairFedDrift can detect concept drift in the third timestep, it can be observed in the fourth timestep that clients in the green concept in FedDrift prefer to associate themselves with the original model (`0'), while FairFedDrift's clients in the green concept associate themselves with models created in the previous timestep. While FairFedDrift's approach to associating clients according to the sum of losses (line 8 in Algorithm \ref{algo:fair-fed-drift}) is important in this scenario, FedDrift lacks this mechanism, resulting in clients preferring the original model.

Furthermore, FairFedDrift effectively merges models corresponding to the same underlying concept (e.g., `1-3-4-5'). This behavior is further illustrated in Figure \ref{fig:merging}, which visualizes the model merge process over time. While Figure \ref{fig:clusters} highlights per-client model associations across timesteps, Figure \ref{fig:merging} offers an additional view of how FairFedDrift merges models. This visual analysis demonstrates FairFedDrift's ability to achieve more accurate and adaptive clustering, thereby offering a robust solution to the problem of group-specific distributed concept drift.

Finally, FairFedDrift offers easier tuning in situations of group-specific concept drift compared to FedDrift. The loss of the unprivileged group undergoes significant variations over time, while the overall loss remains relatively stable. This makes tuning FedDrift more challenging, as it may struggle to detect significant loss changes, making the tuning of its overall loss threshold problematic - smaller values result in excessive clusters, while larger values may fail to detect group-specific concept drift altogether.

\subsection{Effect of $\delta_s$ on Fairness and Performance}

We explore the effect of tuning the hyperparameter $\delta_s$, which controls the acceptable loss difference for a group $s$ between two consecutive timestamps on a global model. $\delta_s$ also acts as a threshold for merging similar models, as described in Algorithm \ref{algo:fair-fed-drift}. The results of different values of $\delta_s$ are presented in Figure \ref{fig:delta_s}.

For the \texttt{Adult-GDrift} dataset, smaller values of $\delta_s$ (between 0.05 and 0.5) work best. In this case, concept drift occurs only through changes in specific smaller subgroups, leading to smoother variations in the loss. In contrast to the other datasets, where drift involves more noticeable label swaps, the drift in \texttt{Adult-GDrift} is more subtle. As a result, smaller $\delta_s$ thresholds are better suited to detect and respond to these drifts. On the other hand, for the image datasets (\texttt{MNIST-GDrift}, \texttt{FEMNIST-GDrift}, and \texttt{CIFAR-10-GDrift}), larger thresholds (around 1.25 to 1.75) are more appropriate. In this case, a higher $\delta_s$ prevents over-sensitivity to minor loss variations and stabilizes clustering.

Finally, with respect to fairness-utility trade-offs, it can be seen that when $\delta_s$ is very small, FairFedDrift may create too many clusters (even one per client) which reduces both fairness and accuracy due to insufficient data per model. As $\delta_s$ increases, clustering becomes more meaningful, leading to improvements in both metrics. However, if $\delta_s$ is set too high (e.g., \texttt{MNIST-GDrift} with $\alpha=0.05$ and $\delta_s=2.0$), drift may remain undetected, causing all clients to converge to a single model - effectively reducing the method to standard FedAvg. In such cases, utility may remain high due to larger training data, but fairness deteriorates, particularly when the unprivileged group is underrepresented.

\subsection{Effect of Window Size on Fairness and Performance}

Finally, we introduce and examine the effect of window size on the performance of our framework. As explained before, the window size parameter determines how much past information is used during model assignment and concept drift detection. To evaluate the impact of window size, we conducted experiments across different window size configurations, ranging from retaining the entire history of client identities and local datasets to utilizing only a sliding window of the latest timesteps. 

Figure \ref{fig:window} illustrates the results of this experiment. It can be observed that in this case maintaining a window size of 3 was sufficient, as increasing the window size beyond this had minimal impact on both fairness and performance across different configurations. Even with a reduced history of data retention, the system consistently achieved comparable levels of fairness and accuracy, demonstrating that a small window size strikes a good balance between computational efficiency and effective drift detection without sacrificing performance.

\begin{figure}[h!]
    \centering
    \includegraphics[width= 0.9\columnwidth]{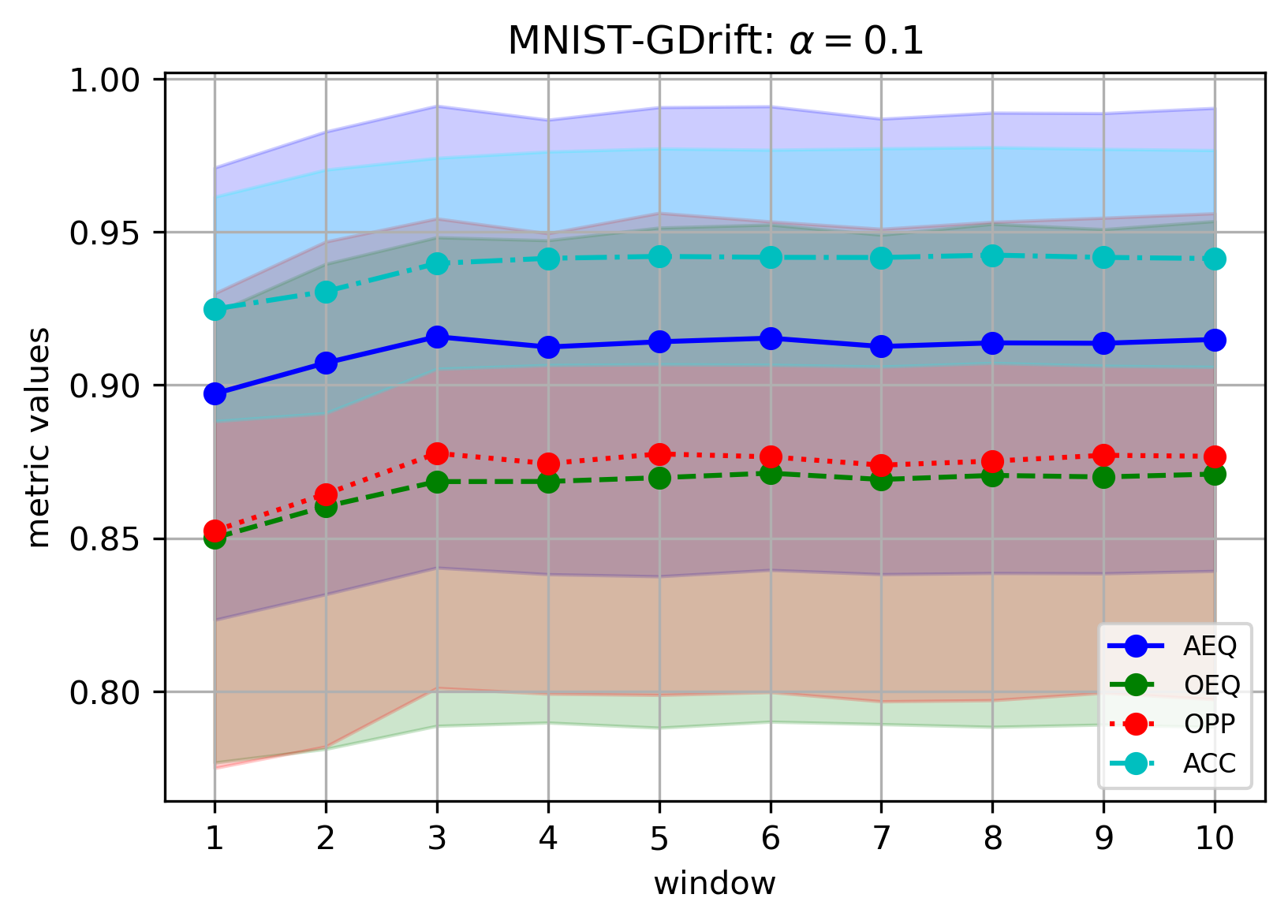}
    \caption{Effect of window size on fairness and performance on the \texttt{MNIST-GDrift} dataset considering $\alpha=0.1$ and Scenario 4.1.}
    \label{fig:window}
\end{figure}

Given the negligible impact of window size on fairness and performance, it is possible to mitigate the overhead associated with preserving the entire history of client identities and local datasets -  clients can opt to maintain only a sliding window of the latest timesteps. By adopting this approach, clients can significantly reduce the computational and storage burden without compromising fairness or performance. Moreover, in certain scenarios, forgetting older data might actively aid in adapting to concept drift, enhancing the system's resilience to changing data distributions over time.

\subsection{Fairness and Performance Across Varying Scenario Complexity}

We analyze the results across two different scenarios (4.1 and 4.5 in Figure \ref{fig:scenarios}), which vary in complexity and drift patterns. Scenario 4.1 represents a simpler case with only three distinct concepts, while Scenario 4.5 presents a more challenging environment, featuring five distinct concepts that emerge in a less predictable, more random fashion. Some concepts even appear simultaneously for the first time, adding further complexity to the drift detection and adaptation.

\begin{table}[h!]
\centering
\renewcommand{\arraystretch}{1.2}
\setlength{\tabcolsep}{4pt}
\begin{tabular}{ccccc:c}
    \toprule
    \multirow{2}{*}{Scenario} & \multirow{2}{*}{Metric} & \multicolumn{4}{c}{Algorithm} \\
    \cline{3-6}
    & & FedAvg & FedDrift & FairFedDrift & Oracle \\ \midrule
    \multirow{4}{*}{4.1} 
    & AEQ & 0.84+-0.07 & 0.85+-0.09 & \textbf{0.91+-0.08} & 0.94+-0.05 \\ 
    & OEQ & 0.81+-0.06 & 0.82+-0.08 & \textbf{0.87+-0.08} & 0.89+-0.05 \\  
    & OPP & 0.81+-0.06 & 0.82+-0.08 & \textbf{0.88+-0.08} & 0.90+-0.05 \\   
    & ACC & 0.94+-0.01 & \textbf{0.96+-0.02} & 0.94+-0.04 & 0.95+-0.02 \\ \midrule
    \multirow{4}{*}{4.5}
    & AEQ & 0.80+-0.10 & 0.80+-0.09 & \textbf{0.90+-0.09} & 0.92+-0.05 \\ 
    & OEQ & 0.77+-0.09 & 0.78+-0.07 & \textbf{0.86+-0.09} & 0.88+-0.05 \\  
    & OPP & 0.77+-0.10 & 0.77+-0.07 & \textbf{0.86+-0.09} & 0.89+-0.05 \\   
    & ACC & 0.94+-0.02 & \textbf{0.95+-0.02} & 0.93+-0.04 & 0.94+-0.03 \\ \bottomrule
\end{tabular}
\caption{Results for two distinct scenarios presented in Figure \ref{fig:scenarios} on the \texttt{MNIST-GDrift} dataset and $\alpha=0.1$.}
\label{table:res-scenarios}
\end{table}

The results, summarized in Table \ref{table:res-scenarios}, illustrate that while Scenario 4.1 was the easiest to handle due to its predictable concept changes, Scenario 4.5 was the most difficult due to the unpredictability and overlap of drifts. Despite this complexity, FairFedDrift consistently performed better compared to FedAvg and FedDrift in terms of fairness across both scenarios, maintaining a competitive accuracy level. Notably, even in the most challenging setting of Scenario 4.5, FairFedDrift maintained a high level of fairness. Moreover, FairFedDrift's fairness and performance in both scenarios closely approached that of the idealized Oracle algorithm.

These results underscore the effectiveness of the proposed method in handling both simple and complex drift scenarios, demonstrating FairFedDrift’s capability to preserve fairness without significant accuracy loss, even under challenging conditions.

\section{Discussion and Future Work}\label{sec:discussion}

This paper introduces and formalizes the concept of group-specific distributed concept drift within FL, providing a framework to address fairness challenges in scenarios with evolving data distributions. While significant progress has been made in advancing fairness-aware FL, several directions remain open for further exploration, particularly in improving the scalability, efficiency, and practical deployment of FairFedDrift.

\paragraph{Automatic Hyperparameter Tuning} One promising area for future work is the automation of the hyperparameter selection process for group-specific thresholds. In its current form, these hyperparameters require manual tuning, which may limit the generalizability of the algorithm to different datasets and application scenarios. Developing adaptive methods for these parameters would enhance the flexibility and robustness of FairFedDrift, reducing the need for expert intervention and allowing for broader applicability across diverse use cases.

\paragraph{Fairness Metrics and Evaluation under Concept Drift} Existing fairness metrics do not fully account for the subtleties of concept drift in FL settings. There is a need to develop specialized metrics that can capture the evolving fairness dynamics as data distributions drift over time. These dynamic metrics should track fairness over time, providing insights into how well fairness-aware algorithms such as FairFedDrift can adapt to evolving conditions and ensure equitable outcomes across different groups.

\paragraph{Scalability and Communication Overhead} A key challenge for the real-world application of FairFedDrift is ensuring scalability and minimizing communication overhead. Monitoring group-specific losses and managing multiple global models can introduce significant computational and communication costs, particularly in large-scale federated systems. To address this, one potential approach is to use gradient similarity techniques to identify models with similar learning patterns and parameters. Additionally, developing novel mechanisms for dynamically selecting the best global model for each client based on their data, such as a cluster assignment model, could further enhance efficiency. Future research should also focus on analyzing various metrics to evaluate latency, communication efficiency, and computational overhead, considering the time, data exchanged, and computational resources required for each operation described in Section \ref{sec:fairfeddrift}.

\paragraph{Real-World Federated Learning Datasets with Concept Drift} To further validate the robustness and applicability of FairFedDrift, future research should focus on real-world datasets that embody spacial and temporal concept drift. Developing and using datasets from sensitive domains that capture the complexities of group-specific concept drift in decentralized environments would be important for a realistic evaluation of FairFedDrift's effectiveness in addressing fairness challenges in real-world FL settings.

\section{Conclusions}\label{sec:conclusions}

This work presents a pioneering effort in introducing and formalizing the problem of group-specific concept drift within FL frameworks. Our experiments underscore the detrimental impact of this phenomenon on fairness, emphasizing the need for effective mechanisms to address evolving data distributions while upholding fairness principles. Our contributions include the formalization of group-specific concept drift and its distributed counterpart, an experimental framework for studying it, and the application of model clustering and fairness drift monitoring techniques to mitigate the challenges posed by this issue. 

We demonstrate that in the absence of group-specific distributed concept drift handling, fairness decreases. In addition, for scenarios with high group imbalance, our findings indicate that overall loss and accuracy do not change significantly. However, there is a noticeable decrease in fairness due to the changing loss of the unprivileged group. In contrast, as the representation of the unprivileged group increases, both overall loss and the loss of the unprivileged group exhibit more noticeable changes. Our proposed algorithm, FairFedDrift, emerged as a practical solution for achieving fairness in FL amid group-specific distributed concept drift. By continuously monitoring group-specific loss and adapting to changes in data distributions, FairFedDrift demonstrated its effectiveness in mitigating unfairness.

The intricate issues of fairness and group-specific distributed concept drift demand ongoing exploration and refinement. Our work lays the foundation for addressing these challenges as machine learning continues to evolve in sensitive domains.

\section*{Ethical Implications}

While our research aims to achieve fairness, it is essential to recognize that this objective may not be guaranteed in all use cases. The fairness metrics selected for evaluation and the choice of sensitive attributes are often contingent upon the application domain. Moreover, the binarization process employed to represent different groups of sensitive attributes may not capture the diversity within each group. These considerations underscore the complexity of fairness in machine learning and highlight the importance of continual improvement to address evolving ethical concerns.

\ifCLASSOPTIONcaptionsoff
  \newpage
\fi



%\bibliographystyle{IEEEtran}

\bibliographystyle{IEEEtran}
\bibliography{bib.bib}

\begin{IEEEbiography}[{\includegraphics[width=1in,height=1.25in,clip,keepaspectratio]{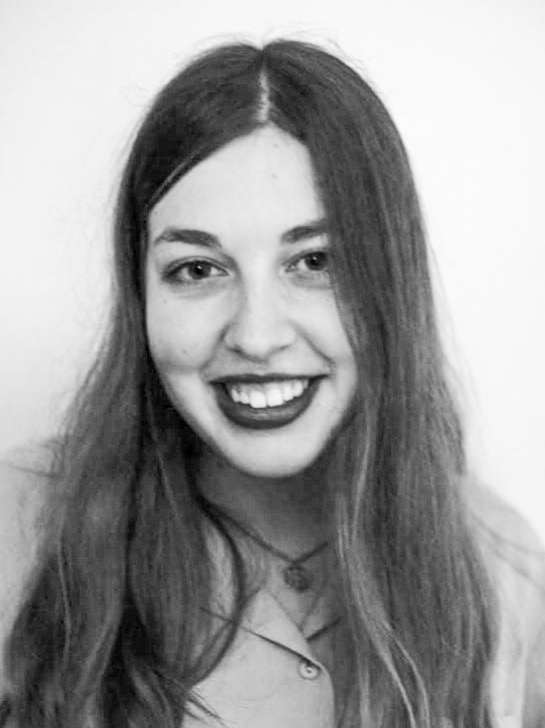}}]
{Teresa Salazar} is a PhD candidate and a member of the Centre for Informatics and Systems at the University of Coimbra. She received a B.S. degree in Informatics Engineering from the same university in 2018 and a M.S. degree in Informatics from the University of Edinburgh in 2019 with specialization in Machine Learning. Her primary research interests are mainly focused on topics in fairness, imbalanced data, and federated learning.
\end{IEEEbiography}

\begin{IEEEbiography}[{\includegraphics[width=1in,height=1.25in,clip,keepaspectratio]{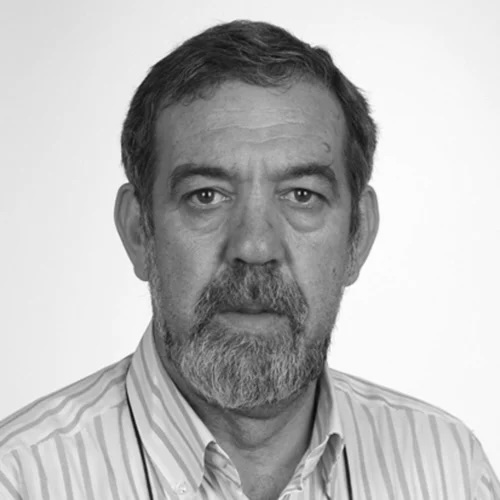}}]
{Jo\~ao Gama} is a Full Professor at the University of Porto, Portugal. He received his PhD in Computer Science from the University of Porto in 2000. He is a EurIA Fellow, IEEE Fellow, Fellow of the Asia-Pacific AI Association, and a member of the board of directors of the LIAAD-INESC TEC. He is an ACM Distinguish Speaker. He is an Editor of several top-level Machine Learning and Data Mining journals. His main research interests are knowledge discovery from data streams, evolving network data, probabilistic reasoning, and causality. He has an extensive list of publications in data stream learning.
\end{IEEEbiography}

\begin{IEEEbiography}[{\includegraphics[width=1in,height=1.25in,clip,keepaspectratio]{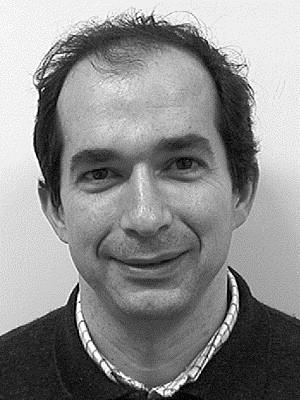}}]
{Helder Ara\'ujo} is a Full Professor at the Department of Electrical and Computer Engineering of the University of Coimbra. His research interests include computer vision applied to robotics, robot navigation and visual servoing. He has also worked on non central camera models, including aspects related to pose estimation, and their applications. Recently he has started work on the development of vision systems applied to medical endoscopy with focus on capsule endoscopy.
\end{IEEEbiography}

\begin{IEEEbiography}[{\includegraphics[width=1in,height=1.25in,clip,keepaspectratio]{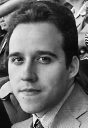}}]
{Pedro Henriques Abreu (PHA)} is an Associate Professor with Habilitation at the Department of Informatics of the University of Coimbra in Portugal, full member of the Cognitive and Media Systems of CISUC, and currently the coordinator of the Master of Informatics Engineering. He is also one of the area editors of the Information Fusion journal and one of the editors of the Data Science and Analytics journal. His research interests include the development of Data Centric AI approaches specially related to missing and imbalance data and data fairness. He is the author of more than 100 refereed journal and conference papers in these areas, and his peer-reviewed publications have been nominated and awarded multiple times as best papers. 
\end{IEEEbiography}

\appendix
\section*{Complexity Analysis} 

To evaluate the computational complexity of FairFedDrift, we compare its key steps with those of the standard FedAvg algorithm. The table below provides a breakdown of computational complexity at different stages of the training process.  

\begin{table}[h]
\centering
\renewcommand{\arraystretch}{1.3}
\begin{tabular}{|p{1.0cm}|p{0.8cm}|p{1.6cm}|p{1.6cm}|p{1.9cm}|}
\hline
\textbf{Step} & \textbf{FedAvg Cost} & \textbf{FedAvg \newline Details} & \textbf{FairFedDrift Cost} & \textbf{FairFedDrift Details} \\  
\hline
Model Sending & $\mathcal{O}(K)$ & The server sends a global model to each of the $K$ clients. & $\mathcal{O}(MK)$ & The server sends $M$ models to each of the $K$ clients. \\  
\hline
Selecting the Best Model & N/A & - & $\mathcal{O}(MSK)$ & Each client evaluates the loss of $S$ groups on each of the $M$ models. \\  
\hline
Merging & N/A & - & $\mathcal{O}(M^2 \log M)$ & The server creates a distance $M \times M$ matrix based on losses of $S$ groups. Then, clustering is applied. \\  
\hline
Local Training & $\mathcal{O}(K)$ & Each client trains a model. & $\mathcal{O}(MK)$ & Each client trains each model $M'$ it has been associated with until that timestep. \\  
\hline
Model Sending to Server & $\mathcal{O}(K)$ & Each client sends its trained model to server. & $\mathcal{O}(M'K)$ & Each client sends its trained $M'$ models to server. \\  
\hline
Model Averaging & $\mathcal{O}(1)$ & The server averages models from $K$ clients. & $\mathcal{O}(M)$ & The server averages $M$ models sent by clients. \\  
\hline
\end{tabular}
\caption{Computational complexity comparison of FedAvg and FairFedDrift for each round within a timestep. $K$ is the number of clients, $M$ is the number of global models, $S$ is the number of sensitive groups.}
\label{tab:complexity_analysis}
\end{table}

Regarding communication costs, the total communication cost over the entire training process can be expressed as follows. In FedAvg, each round involves sending and receiving one model per client, leading to a total of  \(\mathcal{O}(K \times R \times T)\) model transmissions across \(R\) rounds and \(T\) timesteps. In contrast, FairFedDrift’s multi-model approach results in  \(\mathcal{O}(M \times K \times R \times T)\) transmissions, as each of the \(M\) models is communicated to and from each client in every round.

Regarding computation costs, FairFedDrift requires additional computation compared to FedAvg due to its multi-model approach. Unlike FedAvg, where each client trains and sends a single model, FairFedDrift clients select the best model among \(M\) candidates, train all \(M'\) models they have been associated with locally, and the server averages all \(M\) global models instead of just one. Note that in training, \(M'\) refers to the number of models a client has been associated with, which may be fewer than the total number of global models. Moreover, the server also needs to perform global model merging, which involves computing a distance matrix of size \(M \times M\) based on group-specific losses and then applying a clustering algorithm.

This breakdown confirms that while FairFedDrift introduces higher theoretical costs, these remain practical in real-world scenarios. For example, the number of models $M$ is upper-bounded by the number of distinct behavioral patterns. Similarly, the number of sensitive groups $S$ is generally limited, as clients typically contain only a small number of demographic subpopulations. Importantly, these additional costs are a necessary trade-off to effectively capture and adapt to group-specific distributed concept drift, which is essential for ensuring fairness across heterogeneous client populations. Despite this, the sliding window mechanism discussed earlier can mitigate excessive memory and computation overhead by restricting historical data retention.

\section*{Theoretical Guarantees of FairFedDrift} 

Consider a FL setting with a set of clients \( k \in K \), timesteps \( t \in T \), and sensitive groups \( s \in S \). At each timestep \( t \), client \( k \) receives new data \( D_k^t \), sampled from the local distribution \( P_k^t \). The distribution \( P_k^t \) is generated from one of a finite set of \( M \) distinct concepts \( \{C_1, \ldots, C_M\} \). Specifically, the distribution at timestep \( t \) corresponds to concept \( C_{m(t)} \), where  
\[
m : T \to \{1, \ldots, M\}
\]
is a (potentially non-stationary and recurrent) mapping indicating the active concept at time \( t \). That is, concepts may change over time and can reoccur multiple times.

\vspace{0.5em}
\noindent
Consider two algorithms:
\begin{itemize}
    \item \textsc{FairFedDrift}: adapts to group-specific drift using group-specific drift detection and multiple global models $GM$, where each $\theta_m \in GM$ is specialized in concept \( C_m \);
    \item \textsc{NaiveFed}: uses a single global model $\theta_{\text{naive}}$ without group-specific drift detection.
\end{itemize}

\vspace{0.5em}
\noindent

Let \( \theta \) denote a set of model parameters. We define $\ell(\theta)^{D_k^{t}}_{s}$ as the loss for group $s$ and client $k$ on model $\theta_m$ at timestep $t$. To evaluate fairness at the client level, we define a \textit{group fairness disparity} metric for a given client \( k \in K \), model \( \theta \), and timestep \( t \in T \) as:
\[
\Delta_k^t(\theta) := \max_{s, s' \in S} \left| \ell(\theta)^{D_k^{t}}_s - \ell(\theta)^{D_k^{t}}_{s'} \right|
\]
This metric captures the \textit{worst-case disparity in loss across sensitive groups} within a client's local distribution. A lower value of \( \Delta_k^t(\theta) \) indicates \textit{more equitable treatment of groups}, aligning with group fairness notions, and is thus used as a measure of fairness in our comparison of \textsc{FairFedDrift} to \textsc{NaiveFed}.

\begin{proposition}[Group-Specific Drift Detection]
\label{prop:group-loss-drift} At each timestep \( t \), a client \( k \in K \) selects a model \( \theta_{m} \in GM \) from the current set of global models \( GM \) such that:
\begin{itemize}
    \item For all sensitive groups \( s \in S \), the group loss under \( \theta_{m} \) does not exceed the group's loss at the previous timestep by more than the drift threshold \( \delta_s \), i.e.,
    \[
    \ell(\theta_{m})^{D_k^{t}}_{s} \leq \ell(\theta_{w_k^{t-1}})^{D_k^{t-1}}_{s} + \delta_s
    \]
    \item Among all models satisfying the above condition, \( \theta_{w_k^t} \) minimizes the total loss across groups:
    \[
    \theta_{w_k^t} = \arg\min_{\theta_m \in GM} \sum_{s \in S} \ell(\theta_m)^{D_k^t}_s
    \]
\end{itemize}

If no model in \( GM \) satisfies the threshold condition for all groups \( s \in S \), i.e.,
\[
\forall \theta_m \in GM, \exists s \in S \text{ such that } \ell(\theta_m)^{D_k^{t}}_s > \ell(\theta_{w_k^{t-1}})^{D_k^{t-1}}_s + \delta_s,
\]
then the client flags this as concept drift, and a new specialized global model \( \theta_{\text{new}} \) is instantiated, added to \( GM \) and trained on the data from the new distribution.

\medskip

\textbf{Remark} (On the Role of $\delta_s$): In \textsc{FairFedDrift}, each $\delta_s$ is used exclusively as a drift detection threshold. Specifically, $\delta_s$ determines whether a model continues adequately represents a group $s$. If the group-specific loss increases by more than $\delta_s$ between consecutive timesteps, the algorithm flags this as concept drift and instantiates a new specialized model. Because real-world group-specific concept drift can induce arbitrarily large shifts in $P(y \mid X, S=s)$, it is not possible to bound group loss purely in terms of $\delta_s$. Thus, \textsc{FairFedDrift} leverages $\delta_s$ to detect distributional changes early, enabling responsive specialization that mitigates long-term group disparity.

\end{proposition}

\begin{assumption}[Drift Detection and Learnability]
\label{assumption:learnability-expressivity-drift}
Let \( k \in K \), \( s \in S \), and \( t \in T \) be such that client \( k \) experiences group-specific concept drift for group \( s \) at timestep $t$:
\[
P^t_k(y \mid X, S = s) \neq P^{t-1}_k(y \mid X, S = s)
\]
The concept active at timestep \( t \), denoted \( C_{m(t)} \), remains active over a window of \( \tau > 0 \) timesteps, i.e.,
\[
\forall t' \in [t, t+\tau), \quad m(t') = m(t)
\]

Assume the following:

\begin{enumerate}
    \item \textbf{(Drift Detection)} Under an appropriate threshold $\delta_s$, \textsc{FairFedDrift} detects group-specific drift at timestep \( t \), as per Proposition~\ref{prop:group-loss-drift}, and instantiates a new specialized model \( \theta_{\text{new}} \);

    \item (\textbf{Learnability}) There exists a model \( \theta_{\text{new}} \), trained on post-drift data, such that its group loss satisfies:
    \[
    \ell(\theta_{\text{new}})_s^{D_k^{\text{post}}} 
    \leq 
    \inf_{\theta \in \Theta} \ell(\theta)_s^{D_k^{\text{post}}} + \varepsilon,
    \]
    for some small constant \( \varepsilon > 0 \) capturing estimation error, where \( D_k^{\text{post}} \) denotes client \( k \)'s dataset consisting of samples drawn from the post-drift distribution \( P_k^{\text{post}} \) over the interval \( [t, t + \tau) \).

\end{enumerate}
\end{assumption}

\begin{proposition}[Advantage of Specialized Models]
\label{prop:specialization}

Under the assumptions above, there exists a non-zero disparity gap \( \blacktriangle > 0 \) such that the cumulative group fairness disparity on client \( k \) over the interval \( [t, t+\tau) \) satisfies:
\[
\sum_{t' = t}^{t+\tau-1} \Delta_k^{t'}(\theta_{\text{new}}) + \blacktriangle \leq \sum_{t' = t}^{t+\tau-1} \Delta_k^{t'}(\theta_{\text{naive}})
\]
Hence,
\[
\sum_{t' = t}^{t+\tau-1} \Delta_k^{t'}(\theta_{\text{new}}) < \sum_{t' = t}^{t+\tau-1} \Delta_k^{t'}(\theta_{\text{naive}})
\]

Here, the gap \( \blacktriangle \) captures the disparity penalty incurred by \( \theta_{\text{naive}} \), and is lower bounded by a function of:
\begin{itemize}
    \item The inability of \( \theta_{\text{naive}} \) to represent multiple concepts' distributions simultaneously (i.e., representational conflict),
    \item The divergence between new and older distributions,
    \item The irreducible loss gap between a specialized and a mixed-model fit.
\end{itemize}

Consequently, \textsc{FairFedDrift} promotes fairness by ensuring more equitable outcomes across sensitive groups affected by concept drift.

\end{proposition}

\begin{proof} Let \( D_k^{\text{post}} \) denote the dataset consisting of samples drawn from the post-drift distribution \( P_k^{\text{post}} \) over the interval \( [t, t + \tau) \). Similarly, let \( D_k^{\text{prior}} \) denote samples from the pre-drift distribution \( P_k^{\text{prior}} \).

We consider a simplified mixture distribution:
\[
P_k^{\text{mix}} := \alpha \cdot P_k^{\text{post}} + (1 - \alpha) \cdot P_k^{\text{prior}},
\]
where \( \alpha \in (0, 1) \) represents the proportion of post-drift data in the combined dataset \( D_k^{\text{mix}} \).

The naive model \( \theta_{\text{naive}} \) is trained to minimize loss over \( D_k^{\text{mix}} \), not over \( D_k^{\text{post}} \). Due to group-specific concept drift, the conditional distributions \( P(y \mid X, S = s) \) differ significantly between \( D_k^{\text{post}} \) and \( D_k^{\text{prior}} \), inducing a representational conflict:
\[
\ell(\theta_{\text{naive}})_s^{D_k^{\text{post}}} 
\gg  
\inf_{\theta \in \Theta} \ell(\theta)_s^{D_k^{\text{post}}},
\]
due to the inability of a single model to simultaneously perform well on both pre- and post-drift distributions for group \( s \).

By contrast, the specialized model \( \theta_{\text{new}} \), instantiated by \textsc{FairFedDrift} upon detecting drift, is trained directly on \( P_k^{\text{post}} \), and achieves:
\[
\ell(\theta_{\text{new}})_s^{D_k^{\text{post}}} 
\leq 
\inf_{\theta \in \Theta} \ell(\theta)_s^{D_k^{\text{post}}} + \varepsilon
\]

Thus, for all \( t' \in [t, t+\tau) \), the specialized model incurs lower group loss and hence smaller disparity:
\[
\Delta_k^{t'}(\theta_{\text{new}}) + \Delta^\star \leq \Delta_k^{t'}(\theta_{\text{naive}}),
\]
for some constant \( \Delta^\star > 0 \). Summing over all \( t' \in [t, t+\tau) \), we obtain:
\[
\sum_{t'=t}^{t+\tau-1} \Delta_k^{t'}(\theta_{\text{new}}) + \blacktriangle \leq \sum_{t'=t}^{t+\tau-1} \Delta_k^{t'}(\theta_{\text{naive}}),
\]
where \( \blacktriangle = \tau \cdot \Delta^\star > 0 \). Hence,
\[
\sum_{t' = t}^{t+\tau-1} \Delta_k^{t'}(\theta_{\text{new}}) < \sum_{t' = t}^{t+\tau-1} \Delta_k^{t'}(\theta_{\text{naive}}).
\]

This terminates the proof.

\end{proof}

\begin{theorem}[Cumulative Group Fairness Disparity Comparison]
\label{thm:cumulative-loss}
Let client \( k \in K \) experience \( N \) group-specific concept drifts over timesteps \( T \), where each drift satisfies Assumption~\ref{assumption:learnability-expressivity-drift}. Let \( \theta_{w_k^t} \) be the model selected by \textsc{FairFedDrift} at timestep \( t \), and let \( \theta_{\text{naive}} \) be the static model used by \textsc{NaiveFed} throughout. Then, there exists a constant \( \blacktriangle > 0 \) such that:
\[
\sum_{t \in T} \Delta_k^{t}(\theta_{w_k^t}) + \blacktriangle N 
\leq \sum_{t \in T} \Delta_k^{t}(\theta_{\text{naive}})
\]

Since \( N > 0 \) and \( \blacktriangle > 0 \):
\[
\sum_{t \in T} \Delta_k^{t}(\theta_{w_k^t}) < 
\sum_{t \in T} \Delta_k^{t}(\theta_{\text{naive}})
\]
\end{theorem}

By maintaining specialized models that reduce loss disparities across sensitive groups, \textsc{FairFedDrift} ensures that no group experiences disproportionate performance degradation over time. As a result, it provides stronger cumulative fairness guarantees in non-stationary environments than \textsc{NaiveFed}, which cannot adapt to group-specific distributional changes.

\begin{corollary}[Global Cumulative Group Fairness Disparity Comparison]
\label{cor:multi-client-loss}

Under the conditions of Theorem~\ref{thm:cumulative-loss} holding for each client \( k \in K \), the cumulative fairness disparity bound:
\[
\boxed{
\sum_{k \in K} \sum_{t \in T} \Delta_k^{t}(\theta_{w_k^{t}})
<
\sum_{k \in K} \sum_{t \in T} \Delta_k^{t}(\theta_{\text{naive}})
}
\]
continues to hold even when new concepts emerge simultaneously at multiple clients and the global model merging procedure is applied.

\medskip

\textbf{Remark:} The merging operation consolidates global models corresponding to the same concept across clients by leveraging a distance metric that respects group-specific loss thresholds \(\delta_s\). Intuitively, the merging procedure only unifies models whose group-specific losses are sufficiently close, ensuring that merging does not increase group loss disparities beyond thresholds. This preserves the specialization benefits and the fairness guarantees established in the single-client analysis compared to \textsc{NaiveFed}.

Thus, the fairness improvements achieved by \textsc{FairFedDrift} extend across the entire federated system, maintaining the overall reduction in cumulative group fairness disparity compared to \textsc{NaiveFed}.

\end{corollary}

\end{document}